\documentclass[letterpaper]{article} 
\usepackage{aaai2027}  
\usepackage[hyphens]{url}  
\usepackage{graphicx} 
\usepackage{natbib}  
\usepackage{caption} 
\usepackage{algorithm}
\usepackage{algorithmic}
\usepackage{amsmath}
\usepackage{amssymb}
\usepackage{amsthm}
\usepackage{bm}
\usepackage{amsthm}   
\usepackage{tcolorbox}
\usepackage{subcaption}
\usepackage{newfloat}
\usepackage{listings}
\DeclareCaptionStyle{ruled}{labelfont=normalfont,labelsep=colon,strut=off} 
\floatstyle{ruled}
\newfloat{listing}{tb}{lst}{}
\floatname{listing}{Listing}

\usepackage{booktabs}

\title{Autoregressive Decoding as a Martingale: Characterizing the Stochastic Dynamics of LLM Generation}
\author{
    Xiahui Zhu\textsuperscript{\rm 1}\equalcontrib,
    Zhenyi Qi\textsuperscript{\rm 1}\equalcontrib,
    Wanwei Liu\textsuperscript{\rm 1},
    Haoang Chi\textsuperscript{\rm 1}\corresponding,
    Yang Tan\textsuperscript{\rm 1},
    Xinyu Sun\textsuperscript{\rm 1},
    Ji Wang\textsuperscript{\rm 1}
}
\affiliations{
    \textsuperscript{\rm 1}College of Computer, National University of Defense Technology, Changsha, Hunan, China\\

    zhuxiahui21@nudt.edu.cn

}

\usepackage{tikz}
\usetikzlibrary{trees}
\theoremstyle{definition} 
\newtheorem{definition}{Definition}
\theoremstyle{plain} 
\newtheorem{theorem}{Theorem}

\begin{document}

\maketitle

\begin{abstract}
Large language models (LLMs) have achieved remarkable success across a wide range of language tasks. At their core, autoregressive LLMs are conditional probability models that learn to approximate the next-token distribution given the preceding context. Despite this probabilistic nature, the stochastic dynamics of autoregressive decoding have rarely been investigated from a rigorous probabilistic perspective. In this work, we define the per-step prediction error as the difference between the sampled token indicator and the model's conditional predictive distribution, and prove that the cumulative error process is a martingale with respect to the decoding filtration, implying that its conditional first-order moment is identically zero.
Beyond this first-order property, we derive a closed-form expression for the second-order behavior of the martingale sequence: its predictable quadratic variation reduces to $\sum_t\left(1-p(y_t\mid y_{<t})\right)$, computable from realized token probabilities alone. This identity establishes a direct theoretical connection between sequence fluctuations and predictive uncertainty during autoregressive decoding.
White-box experiments on  Llama-3.x and Qwen2.5 families (1B to 72B parameters) verify that the empirically measured second-order moments closely match the theoretical predictions derived from our martingale analysis. These results provide a probabilistic framework for analyzing autoregressive LLM decoding, reveal intrinsic probabilistic fingerprints of different language models, and establish a rigorous mathematical foundation for LLM-generated text detection and uncertainty quantification.
\end{abstract}

\begin{figure}[t]
    \centering
    \includegraphics[width=\linewidth]{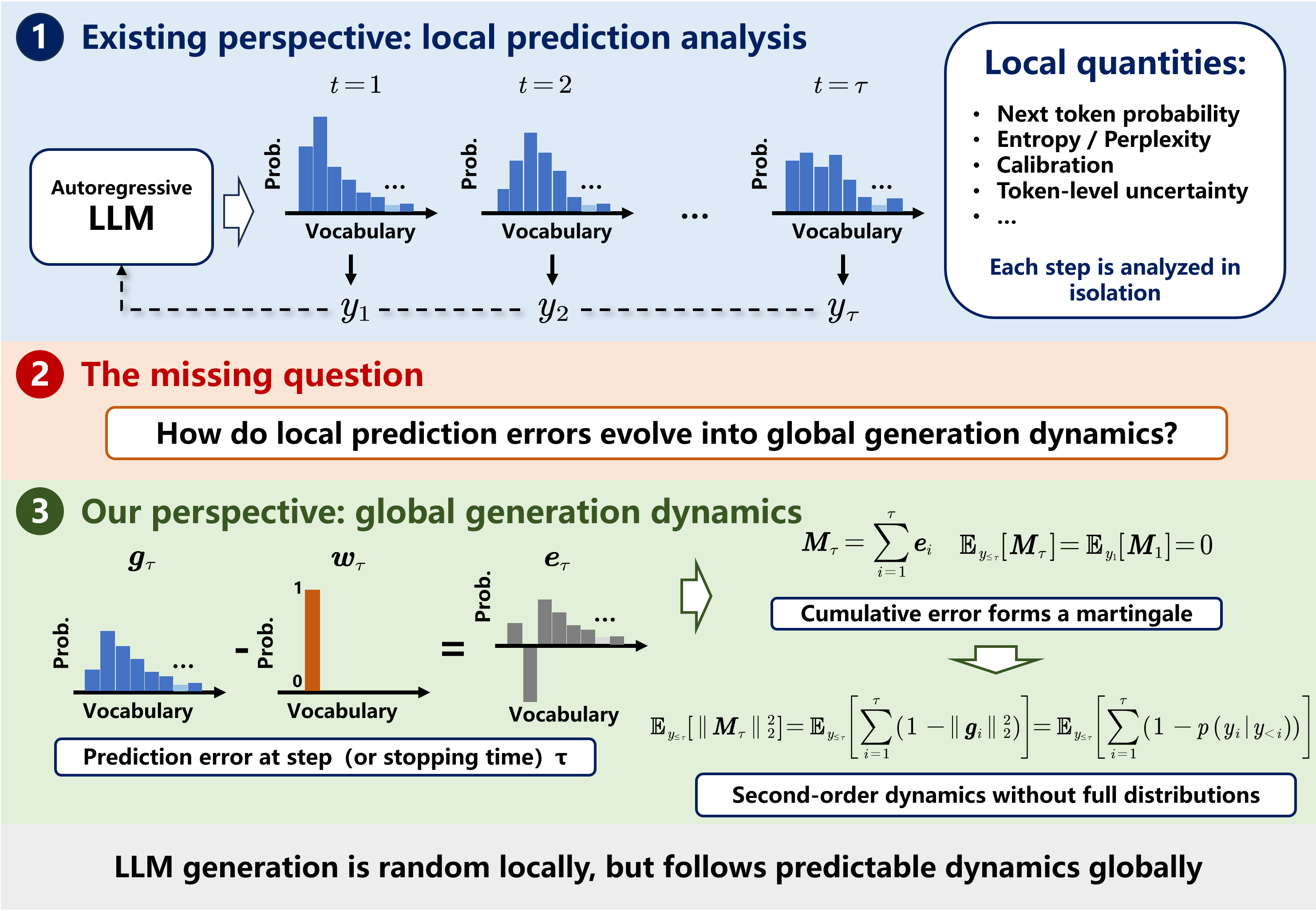}
    \caption{
        \textbf{A stochastic process perspective of autoregressive generation.}
        Existing analyses focus on local token-level uncertainty, while the evolution of prediction errors across the entire generation sequence remains unclear.
        Our perspective models autoregressive decoding as a stochastic process, where cumulative prediction errors form a martingale and exhibit predictable second-order dynamics without requiring full prediction distributions.
        }
    \label{fig:teaser}
\end{figure}


\section{Introduction}
\label{sec:introduction}

Large Language Models (LLMs) have achieved remarkable success across a wide range of natural language tasks through autoregressive decoding \cite{touvron2023llama, achiam2023gpt4}, where tokens are sequentially sampled from context-dependent predictive distributions. Despite their probabilistic nature, the stochastic dynamics governing the evolution of LLM generation sequences remain insufficiently understood from a rigorous probabilistic perspective \cite{welleck2020neural}.

Since autoregressive LLMs are fundamentally conditional probability models, understanding their generation dynamics is a natural yet underexplored direction. Existing analyses mainly focus on local predictive behaviors, including likelihood, uncertainty, entropy, and calibration metrics \cite{holtzman2020curious, meister2020if, guo2017calibration, desai2020calibration, gal2016dropout}. However, these approaches characterize individual prediction distributions rather than the evolution of stochastic deviations over complete decoding sequences.

A fundamental question therefore remains open: how do stochastic prediction errors evolve during autoregressive generation? Since the predictive distribution at each step depends on previously generated tokens, a principled stochastic-process characterization is needed to understand the intrinsic randomness of language generation.

In this work, we introduce a martingale perspective to analyze autoregressive LLM decoding. We define token-level prediction error as the difference between sampled token indicators and conditional predictive probabilities, and prove that the cumulative error process forms a martingale with respect to the decoding filtration. This establishes a fundamental first-order zero-drift property: although individual token realizations are stochastic, their cumulative deviation remains zero in expectation.

Beyond the first-order property, we characterize the fluctuation dynamics through second-order analysis. We derive an analytical identity showing that the martingale second-order moment can be computed directly from realized token probabilities, avoiding explicit construction of high-dimensional error vectors. This establishes a theoretical connection between sequence-level stochastic fluctuations and predictive distributions in autoregressive decoding.

To validate the theoretical framework, we conduct white-box experiments on multiple LLM families using Monte Carlo sampling. Results show that empirical decoding trajectories closely follow the theoretically predicted first- and second-order behaviors.

The main contributions of this work are summarized as follows:

\begin{itemize}
    \item We formulate autoregressive LLM decoding as a stochastic process and establish that the cumulative prediction error vector forms a martingale with respect to the decoding filtration.
    
    \item We derive first- and second-order moment characterizations of the decoding martingale, including an analytical quadratic variation expression determined by the model's predictive probability distribution.
    
    \item We provide comprehensive white-box empirical verification across multiple modern LLM families and scales, confirming the generality of the derived stochastic properties.
\end{itemize}

\section{Related Work}
\label{sec:related_work}

\paragraph{Analysis of Large Language Model Generation Dynamics.}
Understanding the behavior of autoregressive generation has become an important research direction for large language models. Previous studies have mainly investigated generation quality, uncertainty estimation, degeneration phenomena, and decoding strategies \cite{fan2018hierarchical, li2023contrastive, leviathan2023fast, mitchell2023detectgpt}. Holtzman et al. analyzed neural text degeneration and proposed nucleus sampling to alleviate unreliable low-probability token selection during generation \cite{holtzman2020curious, meister2022typical}. Meister et al. studied the relationship between beam search and probabilistic sequence generation, highlighting the importance of decoding strategies in controlling model behavior \cite{meister2020if}. 

Beyond decoding algorithms, recent studies have explored uncertainty estimation and confidence calibration for neural language models. Token-level likelihood, entropy, calibration metrics, and semantic-level uncertainty estimation approaches have been widely adopted to quantify predictive confidence during generation \cite{guo2017calibration, jiang2021know, kuhn2023semantic, lin2022teaching, kadavath2022language, brier1950verification}. However, these approaches primarily focus on the marginal properties of individual predictions. They characterize uncertainty at isolated decoding steps but do not explicitly model how stochastic deviations accumulate and evolve along the autoregressive generation sequence. In contrast, our work studies the sequential stochastic dynamics of prediction errors and establishes a martingale-based characterization of autoregressive decoding.

\paragraph{Theoretical Understanding of Neural Language Models.}
Recent advances in large language models have motivated increasing efforts toward developing theoretical explanations of their capabilities and behaviors \cite{wei2022emergent, ouyang2022training}. Existing theoretical studies have investigated scaling behavior, optimization properties, and internal mechanisms of large neural language models \cite{garg2022what}. Kaplan et al. revealed empirical scaling laws describing the relationship between model size, dataset size, and performance \cite{kaplan2020scaling}. Hoffmann et al. further refined these observations by studying compute-optimal scaling strategies for large language models \cite{hoffmann2022training}. 

In the realm of interpretability, current mechanistic approaches have made significant strides in uncovering how transformer architectures internally represent and process information \cite{meng2022locating, wang2023interpretability, olsson2022incontext}. While these studies yield profound insights into model parameters and internal activations, the stochastic evolution of the generated sequence over time remains largely underexplored. Our work provides a complementary perspective by characterizing the macroscopic probabilistic dynamics induced during autoregressive inference, entirely independent of specific architectural components.

\paragraph{Martingale Theory in Machine Learning.}
Martingale theory provides a fundamental mathematical framework for analyzing stochastic processes whose future conditional expectation is determined by their current information state. Classical martingale results have played an important role in probability theory, sequential analysis, and stochastic optimization \cite{doob1953stochastic, freedman1975tail, mcdiarmid1989method}. In machine learning, martingale techniques have been widely applied to derive concentration bounds, analyze online learning algorithms, and establish theoretical guarantees for sequential decision-making problems. For example, Azuma-Hoeffding inequalities and related martingale concentration tools have been extensively used in reinforcement learning and multi-armed bandit analysis \cite{azuma1967weighted, bubeck2012regret, lattimore2020bandit}.

\paragraph{Stochastic Process Perspectives in Machine Learning.}
Stochastic process theory has provided powerful tools for analyzing sequential behaviors in machine learning systems. Recent studies have employed stochastic differential equations, diffusion processes, and stochastic optimization frameworks to characterize the dynamics of learning algorithms and generative models \cite{song2021score, li2022diffusion}. However, these analyses mainly focus on training dynamics or continuous generative processes, while the discrete stochastic evolution of autoregressive token generation remains largely unexplored. Our work differs by directly modeling the inference-time decoding sequence of LLMs as a discrete martingale process.

Despite its importance in sequential learning theory, martingale analysis has not been systematically explored for autoregressive language generation. In this work, we bridge this gap by introducing a martingale formulation for the cumulative prediction error process of LLM decoding. This perspective enables rigorous first- and second-order moment analysis of the generation sequence and provides a probabilistic characterization of autoregressive decoding dynamics.

\section{Preliminaries}
\label{sec:preliminaries}

In this section, we briefly review the central concepts of martingale theory within the context of probability theory and stochastic processes, providing a theoretical underpinning for the modeling of the autoregressive generation process discussed in the subsequent chapter.


Consider a complete probability space $(\Omega, \mathcal{F}, P)$. In the context of discrete-time autoregressive text generation, the output token sequence $\{y_t\}_{t \geq 1}$ produced by the stochastic process corresponds to the observed sequences in the sample space.

\begin{definition}[Filtration]
A sequence of sub-$\sigma$-algebras $\mathbb{F} = \{\mathcal{F}_t\}_{t \geq 1}$ on the measure space is called a \textit{filtration} if it satisfies the property of monotonicity, i.e., $\mathcal{F}_1 \subseteq \mathcal{F}_2 \subseteq \cdots \subseteq \mathcal{F}_n \subseteq \mathcal{F}$. In this study, we define $\mathcal{F}_t = \sigma(y_1,y_2, \dots, y_t)$ as the $\sigma$-algebra generated by the historical token sequence up to time $t$, representing all observed historical text information \cite{doob1953stochastic}.
\end{definition}


\begin{definition}[Martingale]
Given a filtration $\mathbb{F}$, a stochastic process $M = \{M_t\}_{t \geq 1}$ adapted to this filtration is called a martingale with respect to $\mathbb{F}$ if it satisfies:

\begin{enumerate}
    \item \textbf{Integrability:} For any $t \geq 1$, $\mathbb{E}[|M_t|] < \infty$;
    \item \textbf{Martingale Property:} For any $t \geq 1$, the conditional expectation equality holds almost surely:
    \begin{equation}
        \mathbb{E}[M_{t+1} \mid \mathcal{F}_t] = M_t.
    \end{equation}
\end{enumerate}
\end{definition}

In the context of text generation, this property implies that if the model prediction is accurate and unbiased, then the expected increment of the future evolution of the cumulative prediction error—where observed values deviate from predicted probabilities—given the historical information, should be zero.


In actual text generation tasks, the length of a sequence is often not fixed but determined by specific stopping conditions (such as generating an [EOS] token or reaching a maximum length limit). This is mathematically characterized by ``stopping time“.

\begin{definition}[Stopping Time]
A random variable $\tau: \Omega\to \mathbb{N}\cup \{+\infty\}$ is called a stopping time with respect to filtration $\mathbb{F}$ if for any $t$, the event $\{\tau = t\}$ depends only on the historical information up to time $t$, i.e., $\{\tau = t\} \in \mathcal{F}_t$ \cite{doob1953stochastic}.
\end{definition}

With the definition of stopping time, we have the theorem as follows.

\begin{theorem}[Martingale Stopping Theorem]
Let $M = \{M_t\}_{t \geq 1}$ be a martingale with respect to $\mathbb{F}$, and $\tau$ be a stopping time. If $\tau$ is bounded (i.e., there exists a constant $K$ such that $\tau \le K$ almost surely), then the expected value of the process at the stopping time equals its expected value at the initial time: $\mathbb{E}[M_\tau] = \mathbb{E}[M_1]$.
\end{theorem}

Noticing that for LLMs the maximum generation length ensures a bounded stopping time, MST conditions are always satisfied. This guarantees that the expected cumulative prediction error remains statistically conserved for model-generated sequences, independent of stochastic termination.

\section{Theoretical Analysis}
\label{sec:methodology}

This section aims to prove and explore martingale properties of LLMs. First, we formalize the generation process of autoregressive language models as a stochastic process, prove that it satisfies the martingale property, and further examine the properties of the first-order moment. Second, based on previous results, we turn our attention to the second-order moment of the martingale, where a series of formulas are given to accommodate the practical operation of LLMs. Finally, we examine the meaning of martingale second-order moments, linking it to the stability of model outputs.

\subsection{Notations}
Let $\bm{X}$ and $\bm{Y} = \{y_t\}_{1\leq t\leq l_{\mathrm{max}}}$ be a discrete sequence of input and output text respectively, where $y_t \in \mathcal{V}$ denotes the token at time step $t$, $l_{\mathrm{max}}$ denotes the maximum generation length of LLM, and $\mathcal{V}$ denotes the complete vocabulary set of the model. We define the filtration $\mathbb{F} = \{\mathcal{F}_t\}_{t \geq 1}$, where $\mathcal{F}_t = \sigma(y_1,y_2, \dots, y_t)$. To be precise, $\mathcal{F}_t$ is the power set of $\mathcal{V}^t$.

The generation process of LLM, an autoregressive model, is essentially a parameterized conditional probability distribution $p(y_t\mid y_{<t} ;\bm{X},\bm{\theta})$, often abbreviated as $p(y_t\mid y_{<t})$ when $\bm{X}$ and $\bm{\theta}$ remain unchanged and are not the focus of discussion, where $y_{<t}$ represents the sequence $y_1y_2\dots y_{t-1}$. Similarly, we donte $y_{\leq t}$ as the sequence $y_1y_2\dots y_t$.

In what follows, the notation $\mathbb{E}_{y_t}[\cdot]$, $\mathbb{E}_{y_{<t}}[\cdot]$ will denote the expectations taken over the random variables $y_t$ and $y_1,y_1,\dots,y_{t-1}$ respectively.  Other cases follow analogously.

In practice, large language models output vectors rather than scalars, so we introduce the following vectors in $\mathbb{R}^{|\mathcal{V}|}$:

\textbf{Predictive Probability Vector} ($\bm{g}_t$):
\begin{equation}
    \bm{g}_{t,y}= p(y\mid y_{<t};\bm{X},\bm{\theta}).
\end{equation}

This vector is just the predictive probability vector given by the model for next token. It depends on input $\bm{X}$, parameters $\bm{\theta}$ and the history of output $y_{<t}$.

\textbf{Token Indicator Vector} ($\bm{w}_t$):
\begin{equation}
    \bm{w}_{t,y}=\mathbb{I}_{\{y=y_t\}}.
\end{equation}

It's the one-hot vector for the token actually generated by the model at time $t$.

\textbf{Prediction Error Vector} ($\bm{e}_t$):
\begin{equation}
    \bm{e}_t=\bm{w}_t-\bm{g}_t.
\end{equation}

Notice that $\mathbb{E}_{y_t}[\bm{w}_t\mid y_{<t}] = \bm{g}_t$, error vector $\bm{e}_t$ actually measures the discrepancy between $\bm{w}_t$ and its expectation $\bm{g}_t$.

\textbf{Cumulative Prediction Error Vector} ($\bm{M}_t$):
\begin{equation}
    \bm{M}_t=\sum_{i=1}^t \bm{e}_i.
\end{equation}

 Just as the name suggests. $\bm{M}_t$ is the accumulation of error $\bm{e}_t$, which is the central focus of our paper. We will proof that it is actually a martingale and ‌elaborate on its various properties.

To more meticulously model the stochastic process underlying text generation by LLM, it is essential to account for the randomness in the length of the generated text. Hence, we introduce the concept of stopping time:

\textbf{Stopping Time} ($\tau$):
A random variable $\tau: \Omega \to \{1,2,\dots,l_{\mathrm{max}}\}$ which is defined as follows:
\begin{equation}
    \tau := \inf\ \{ t > 0 : y_t\mathrm{\ is\ [EOS]} \}\cup \{l_{\mathrm{max}}\}.
\end{equation}

In the context of token tracing, $\tau$ is actually the output length. When [EOS] is generated, $\tau$ is the first occurrence time of the end-of-sequence token. Otherwise it is the maximum generation length $l_{\mathrm{max}}$ with length limit reached. It can be easily verified that $\{\tau = t\}\in \mathcal{F}_t$ and thus $\tau$ is a stopping time with respect to $\mathbb{F}$.

It should be noted that $\tau$ is strictly bounded by the maximum generation length $l_{\mathrm{max}}$, which satisfies the boundedness condition required for martingale stopping theorem.

\subsection{Martingale and its Property of First-Order Moment}

The cornerstone of our entire theoretical analysis rests upon the fact that $\bm{M}_t$ is a martingale with respect to the filtration $\mathbb{F}$.

\begin{theorem}
The stochastic process $\{\bm{M}_t\}_{t \geq 1}$ is a martingale with respect to the filtration $\mathbb{F}$.
\end{theorem}

Furthermore, we define the first-order moment of the martingale as one aspect of the study of its martingale properties.

\begin{definition}[First-Order Moment of the Martingale]
    Given a martingale $\{M_t\}_{t\geq 1}$, the first-order moment of martingale $\{M_t\}_{t\geq 1}$ at time $t$ (or stopping time $\tau$) is $\mathbb{E}_{y_{\leq t}}[\bm{M}_t]$ (or $\mathbb{E}_{y_{\leq l_{\mathrm{max}}}}[\bm{M}_{\tau}]$).
\end{definition}

It is not difficult to see that the first-order moment of the martingale $\bm{M}_t$ in fact characterizes the expected value of the accumulated prediction error. By the martingale expectation property, it follows that 
\begin{equation}
    \mathbb{E}_{y_{\leq t}}[\bm{M}_t]=\mathbb{E}_{y_1}[\bm{M}_1]=\mathbb{E}_{y_1}[\bm{w}_1]-\bm{g}_1=0.
    \label{eq:first-order}
\end{equation}

Furthermore, by martingale stopping theorem, we obtain
\begin{equation}
    \mathbb{E}_{y_{\leq l_{\mathrm{max}}}}[\bm{M}_{\tau}]=\mathbb{E}_{y_1}[\bm{M}_1]=0.
    \label{eq:first-order-stopping-time}
\end{equation}

It should be emphasized that equation~\eqref{eq:first-order-stopping-time} holds for arbitrary $\mathbb{F}$-stopping times, thereby rendering it a stronger version of equation~\eqref{eq:first-order}---indeed, equation~\eqref{eq:first-order} is obtained as the special case where $\tau \equiv t$. In what follows, any theorem or equation stated in terms of stopping times is uniformly understood as a strengthened formulation of the corresponding non-stopping-time version; we shall not repeat this remark subsequently.

Equations \eqref{eq:first-order} and \eqref{eq:first-order-stopping-time} show that whether for fixed time $t$ or stopping time $\tau$, the first-order moment of $\{\bm{M}_t\}_{t\geq 1}$ is always 0, which means accumulated prediction error vanishes to 0 in the long run. This implies that the first-order moment alone cannot furnish sufficient information for analytical and practical purposes, which motivates us to investigate the properties of the second-order moment in the next section.
 
\subsection{Calculation of the Second-Order
Moment of the Martingale}

To gain a deeper characterization of martingale properties, we now introduce the second-order moment of the martingale.

\begin{definition}[Second-Order
Moment of the Martingale]
    Given a martingale $\{M_t\}_{t\geq 1}$, the second-order moment of martingale $\{M_t\}_{t\geq 1}$ at time $t$ (or stopping time $\tau$) is $\mathbb{E}[\|\bm{M}_t\|_2^2]$ (or $\mathbb{E}_{y_{\leq l_{\mathrm{max}}}}[\|\bm{M}_{\tau}\|_2^2]$).
\end{definition}

Regarding the calculation of the second-order moment, we present the following theorem:

\begin{theorem}
    \label{thm:raw-calculate}
    The second-oreder moment of martingale $\{M_t\}_{t\geq 1}$ can be expressed as follows:
    \begin{equation}
        \mathbb{E}_{y_{\leq t}}[\|\bm{M}_t\|_2^2] = \mathbb{E}_{y_{< t}}\left[\sum_{i=1}^t(1 - \|\bm{g}_i\|_2^2)\right].
        \label{eq:M-square-formula}
    \end{equation}
\end{theorem}

Theorem~\ref{thm:raw-calculate} offers a viable approach for computing the second order moment, which is determined exclusively by the vector $\bm{g}_t$, indicating that it constitutes an inherent characteristic of the LLM. However, it doesn't consider the random generation text length $\tau$---it can only calculate for fixed length $t$. Hence, we turn to the formula for the second order moment associated with stopping time. 

\begin{theorem}
    The second order moment of martingale $\{\bm{M}_t\}_{t\geq 1}$ with stopping time can be expressed as follows:
    \begin{equation}
        \mathbb{E}_{y_{\leq l_{\mathrm{max}}}}[\|\bm{M}_{\tau}\|_2^2]=\mathbb{E}_{y_{\leq l_{\mathrm{max}}}}\left[\sum_{t=1}^{\tau}(1-\|\bm{g}_t\|_2^2)\right].
        \label{eq:M-square-formula-terminate}
    \end{equation}
    \label{thm:stoptime-calculate}
\end{theorem}

In contrast to Theorem~\ref{thm:raw-calculate}, Theorem~\ref{thm:stoptime-calculate} better reflects the practical operation of LLMs. Nevertheless, it presents an application-level concern: the individual components of $\bm{g}_t$ are not necessarily all available, depending on the accessibility (open-source vs. closed-source) of the model. In response, we present the following theorem:
\begin{theorem}
    \label{thm:raw-easy-calculate}
    The second-oreder moment of martingale $\{M_t\}_{t\geq 1}$ can be expressed as follows:
    \begin{equation}
        \mathbb{E}_{y_{\leq t}}[\|\bm{M}_t\|_2^2] =\mathbb{E}_{y_{\leq t}}\left[\sum_{i=1}^t(1-p(y_i\mid y_{<i}))\right].
        \label{eq:M-square-formula-easy}
    \end{equation}
\end{theorem}

\begin{theorem}
    The second order moment of martingale $\{\bm{M}_t\}_{t\geq 1}$ with stopping time can be expressed as follows:
    \begin{equation}
        \begin{gathered}
            \mathbb{E}_{y_{\leq l_{\mathrm{max}}}}[\|\bm{M}_{\tau}\|_2^2]\\
            =\mathbb{E}_{y_{\leq l_{\mathrm{max}}}}\left[\sum_{t=1}^{\tau}(1-p(y_t\mid y_{<t}))\right].
        \end{gathered}
        \label{eq:M-square-formula-terminate-easy}
    \end{equation}
    \label{thm:stoptime-easy-calculate}
\end{theorem}

Compared with equations \eqref{eq:M-square-formula} and \eqref{eq:M-square-formula-terminate}, equations \eqref{eq:M-square-formula-easy} and \eqref{eq:M-square-formula-terminate-easy} do not require knowledge of every component of $\bm{g}_t$. They only need the actually generated tokens $y_{\leq t}$ and their associated conditional probabilities $p(y_t\mid y_{<t})$, which are usually accessible from typical LLM inference APIs. In addition, their computation does not require squaring and summing over $|\mathcal{V}|$ components, but only scalar summations. This suggests that equations \eqref{eq:M-square-formula-easy} and \eqref{eq:M-square-formula-terminate-easy} are both more computationally efficient and easier to implement in practice.

\subsection{Second-Order Moment and Stability of Model Outputs}
In fact, the second order moment of martingales is closely tied to the stability of LLM outputs. Here we restrict our attention to syntactic stability, which refers to the similarity of generated texts across repeated generating.
\begin{theorem}
    The second-oreder moment of martingale (with stopping time) can be expressed as follows:
    \begin{gather}
        \mathbb{E}_{y_{\leq t}}[\|\bm{M}_t\|_2^2]=\mathbb{E}_{y_{\leq t}}\left[\sum_{i=1}^t \bm{e}_i^2\right],\label{eq:M-error}\\
        \mathbb{E}_{y_{\leq l_{\mathrm{max}}}}[\|\bm{M}_{\tau}\|_2^2]=\mathbb{E}_{y_{\leq l_{\mathrm{max}}}}\left[\sum_{t=1}^{\tau} \bm{e}_t^2\right].\label{eq:M-error-terminate}
    \end{gather}
    \label{thm:M-error}
\end{theorem}

Theorem~\ref{thm:M-error} implies that the second order moment of martingales captures the expected sum of squared errors at each step, which essentially quantifies the uncertainty in LLM outputs. The larger this value, the less reproducible the model is, and vice versa. It equals zero if and only if the error at every step is zero---that is, the model's generation is entirely determined by the given input and parameters.

The following theorem further characterizes the relationship between the second order moment of martingales and the output stability of the model.

\begin{theorem}
    For the second order moments $\mathbb{E}_{y_{\leq t}}[\|\bm{M}_t\|_2^2]$ and $\mathbb{E}_{y_{\leq l_{\mathrm{max}}}}[\|\bm{M}_T\|_2^2]$, the following relation holds:
    \begin{gather}
         \sum_{i=1}^t\mathrm{Var}(\bm{w}_i)=\mathbb{E}_{y_{\leq t}}[\|\bm{M}_t\|_2^2]+\sum_{i=1}^t\mathrm{Var}(\bm{g}_i),\label{eq:M-decompose}\\
         \begin{gathered}
             \mathbb{E}_{y_{\leq l_{\mathrm{max}}}}\left[\sum_{t=1}^{\tau}\mathrm{Var}(\bm{w}_t)\right]\\
             =\mathbb{E}_{y_{\leq l_\mathrm{max}}}[\|\bm{M}_{\tau}\|_2^2]+\mathbb{E}_{y_{\leq l_{\mathrm{max}}}}\left[\sum_{t=1}^{\tau}\mathrm{Var}(\bm{g}_t)\right].
         \end{gathered}
         \label{eq:M-decompose-terminate}
    \end{gather}

    Then
    \begin{gather}
        \sum_{i=1}^t\mathrm{Var}(\bm{w}_i)\geq \mathbb{E}_{y_{\leq t}}[\|\bm{M}_t\|_2^2],\label{eq:M-lower-bound}\\
         \mathbb{E}_{y_{\leq l_{\mathrm{max}}}}\left[\sum_{t=1}^{\tau}\mathrm{Var}(\bm{w}_t)\right]\geq\mathbb{E}_{y_{\leq l_\mathrm{max}}}[\|\bm{M}_{\tau}\|_2^2].\label{eq:M-lower-bound-terminate}
    \end{gather}
    \label{thm:var-decompose}
\end{theorem}

Theorem~\ref{thm:var-decompose} states that the sum of variances of $\bm{w}_i$ admits a decomposition into two components: the second order moment of the martingale and the sum of variances of $\bm{g}_i$, and the second order moment serves as its lower bound. Specifically, the second order moment measures the randomness within the same output sequence, while the sum of variances of $\bm{g}_i$ measures the randomness across different output sequences. It can be shown that minimizing the second order moment of the martingale is equivalent to minimizing the sum of variances of $\bm{w}_i$.

\begin{theorem}
    The following equivalences hold:
    \begin{gather}
        \sum_{i=1}^t\mathrm{Var}(\bm{w}_i)=0\iff \mathbb{E}_{y_{\leq t}}[\|\bm{M}_t\|_2^2]=0\label{eq:M-zero-error},\\
        \begin{gathered}
           \mathbb{E}_{y_{\leq l_{\mathrm{max}}}}\left[\sum_{t=1}^{\tau}\mathrm{Var}(\bm{w}_t)\right]=0\iff \\
           \mathbb{E}_{y_{\leq l_\mathrm{max}}}[\|\bm{M}_{\tau}\|_2^2]=0. 
        \end{gathered}
        \label{eq:M-zero-error-terminate}
    \end{gather}
    \label{thm:M-var-zero}
\end{theorem}

With Theorem~\ref{thm:M-var-zero}, we naturally ask whether the maximization of the second order moment of the martingale is equivalent to the maximization of the sum of variances of $\bm{w}_i$. The answer is negative. We have the following theorem:

\begin{theorem}
    The following equivalences hold:
    \begin{gather}
        \begin{gathered}
            \sum_{i=1}^t\mathrm{Var}(\bm{w}_i)=t(1-\frac{1}{|\mathcal{V}|})\iff\\
            \forall 1\leq i\leq t:\mathbb{E}_{y_{< i}}[\bm{g}_i]=\frac{1}{|\mathcal{V}|}\mathbf{1},
        \end{gathered}\\
        \begin{gathered}
            \mathbb{E}_{y_{\leq t}}[\|\bm{M}_t\|_2^2]=t(1-\frac{1}{|\mathcal{V}|})\iff \\
            \forall 1\leq i\leq t:(\bm{g}_i=\frac{1}{|\mathcal{V}|}\mathbf{1}\ \mathrm{a.e.}),
        \end{gathered}
    \end{gather}
    where $\mathbf{1}$ is the all-one vector.
    
    The following implication hold:
    \begin{equation}
        \begin{gathered}
            \mathbb{E}_{y_{\leq t}}[\|\bm{M}_t\|_2^2]=t(1-\frac{1}{|\mathcal{V}|})\Rightarrow\\
            \sum_{i=1}^t\mathrm{Var}(\bm{w}_i)=t(1-\frac{1}{|\mathcal{V}|}).
        \end{gathered}
        \label{eq:entail-for-M-mi-var-max}
    \end{equation}
\end{theorem}

Counterexample to the converse of equation~\eqref{eq:entail-for-M-mi-var-max} is provided, as shown in Figure~\ref{fig:counterexample}.

\begin{figure}[htbp]
    \centering
    \begin{subfigure}{0.49\linewidth}
        \centering
        \begin{tikzpicture}[level/.style={level distance=15mm, sibling distance=10mm}]
            \node (root) {root}
                child { node {a(0.5)} 
                    child { node {a(0.7)} }
                    child { node {b(0.3)} }
                }
                child [missing] {}
                child { node {b(0.5)}
                    child { node {b(0.7)} }
                    child { node {a(0.3)} }
                };
        \end{tikzpicture}
        \caption{Maximize $\sum_{i=1}^t\mathrm{Var}(\bm{w}_i)$.}
        \label{fig:max-for-wi-var}
    \end{subfigure}
    \hfill
    \begin{subfigure}{0.49\linewidth}
        \begin{tikzpicture}[level/.style={level distance=15mm, sibling distance=10mm}]
            \node (root) {root}
                child { node {a(0.5)}
                    child { node {a(0.5)} }
                    child { node {b(0.5)} }
                }
                child [missing] {}
                child { node {b(0.5)}
                    child { node {b(0.5)} }
                    child { node {a(0.5)} }
                };
        \end{tikzpicture}
        \caption{Maximize $\mathbb{E}_{y_{\leq t}}[\|\bm{M}_t\|_2^2]$.}
        \label{fig:max-for-M}
    \end{subfigure}
    \caption{The case where $\mathbb{E}_{y_{\leq t}}[\|\bm{M}_t\|_2^2]$ isn't maximized when $\sum_{i=1}^t\mathrm{Var}(\bm{w}_i)$ is maximized.}
    \label{fig:counterexample}
\end{figure}
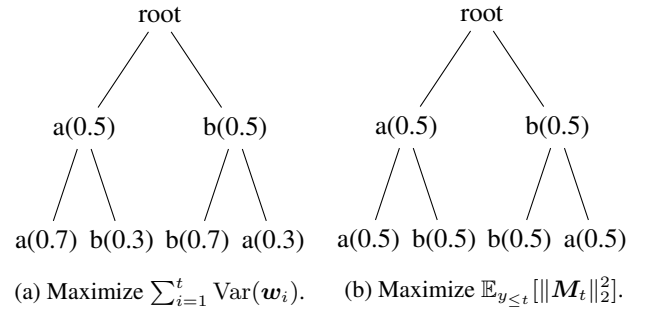

In Figures \ref{fig:max-for-wi-var} and \ref{fig:max-for-M}, each node has $|\mathcal{V}|$ child nodes, representing the tokens that can be generated at the next time step, with each child node associated with a selection probability $p(y_t\mid y_{<t})$.

For Figure~\ref{fig:max-for-wi-var}, we have 
\begin{equation}
    \begin{gathered}
        \mathbb{E}_{y_{\leq 2}}[\|\bm{M}_2\|_2^2]=0.92,\\
        \sum_{i=1}^2\mathrm{Var}(\bm{w}_2)=1.
    \end{gathered}\label{eq:cal-for-M-max-M}
\end{equation}

For Figure~\ref{fig:max-for-M}, we have
\begin{equation}
    \begin{gathered}
        \mathbb{E}_{y_{\leq 2}}[\|\bm{M}_2\|_2^2]=1,\\
        \sum_{i=1}^2\mathrm{Var}(\bm{w}_2)=1.
    \end{gathered}\label{eq:cal-for-mi-var-M-max-M}
\end{equation}

Comparing equation~\eqref{eq:cal-for-M-max-M} with equation~\eqref{eq:cal-for-mi-var-M-max-M}, we can observe that the second order moment of the martingale does not necessarily reach its maximum when the sum of the variances of $\bm{w}_i$ is maximized. Furthermore, it is intuitively expected that Figure~\ref{fig:max-for-M} exhibits greater randomness than Figure~\ref{fig:max-for-wi-var}. Nevertheless, the above calculation results show that the sum of the variances of $\bm{w}_i$ cannot reflect this discrepancy, whereas the second order moment of the martingale can. This suggests that, in contrast to the sum of the variances of $\bm{w}_i$, the second order moment of the martingale serves as a more effective measure of output stability.

\FloatBarrier

\section{Experiments}

\subsection{Experimental Setup}

We conduct white-box experiments (with full logit access) to empirically verify the theoretical properties of our martingale formulation. Computations are performed on an Ubuntu 22.04 system equipped with two NVIDIA RTX PRO 6000 GPUs, a 44-core Intel Xeon Platinum 8470Q CPU, 220GB RAM, Python 3.12, PyTorch 2.12.1, and CUDA 13.0.

By default, we use \textbf{Llama-3.2-3B-Instruct}. The decoding process uses pure stochastic sampling ($T=1.0$, top-$p=1.0$, top-$k=0$) with no repetition penalties or external constraints. The experiments are organized into four parts: (1) verification of the first-order martingale property, (2) verification of the derived second-order moment identity, (3) scaling and generality analysis, and (4) comparison with existing uncertainty measures.

\subsection{First-Order Martingale Verification}

According to Theorem~1 and equation~\eqref{eq:first-order},
the cumulative prediction error process $\mathbf{M}_t$
forms a martingale with zero first-order moment.

Therefore, we empirically examine whether the estimated cumulative prediction error
remains centered around zero during autoregressive generation.

To empirically verify this property, we perform a single stochastic generation sequence with Llama-3.2-3B-Instruct.
At each decoding step, the complete vocabulary-level error vector is computed.
Then, $K=300$ vocabulary dimensions are randomly sampled, and the empirical mean sequence is estimated as
$\hat{\mu}_t=\frac{1}{K}\sum_{j=1}^{K}M_{t,v_j}$.

If the martingale property holds, this sequence should remain centered around zero without systematic drift.

Figure~\ref{fig:first_order} shows the evolution of the empirical first-order moment.
Although individual vocabulary dimensions exhibit random fluctuations,
their averaged sequence remains close to zero throughout the generation process,
which agrees with the theoretical expectation.

\begin{figure}[htbp]
    \centering
    \includegraphics[width=\linewidth]{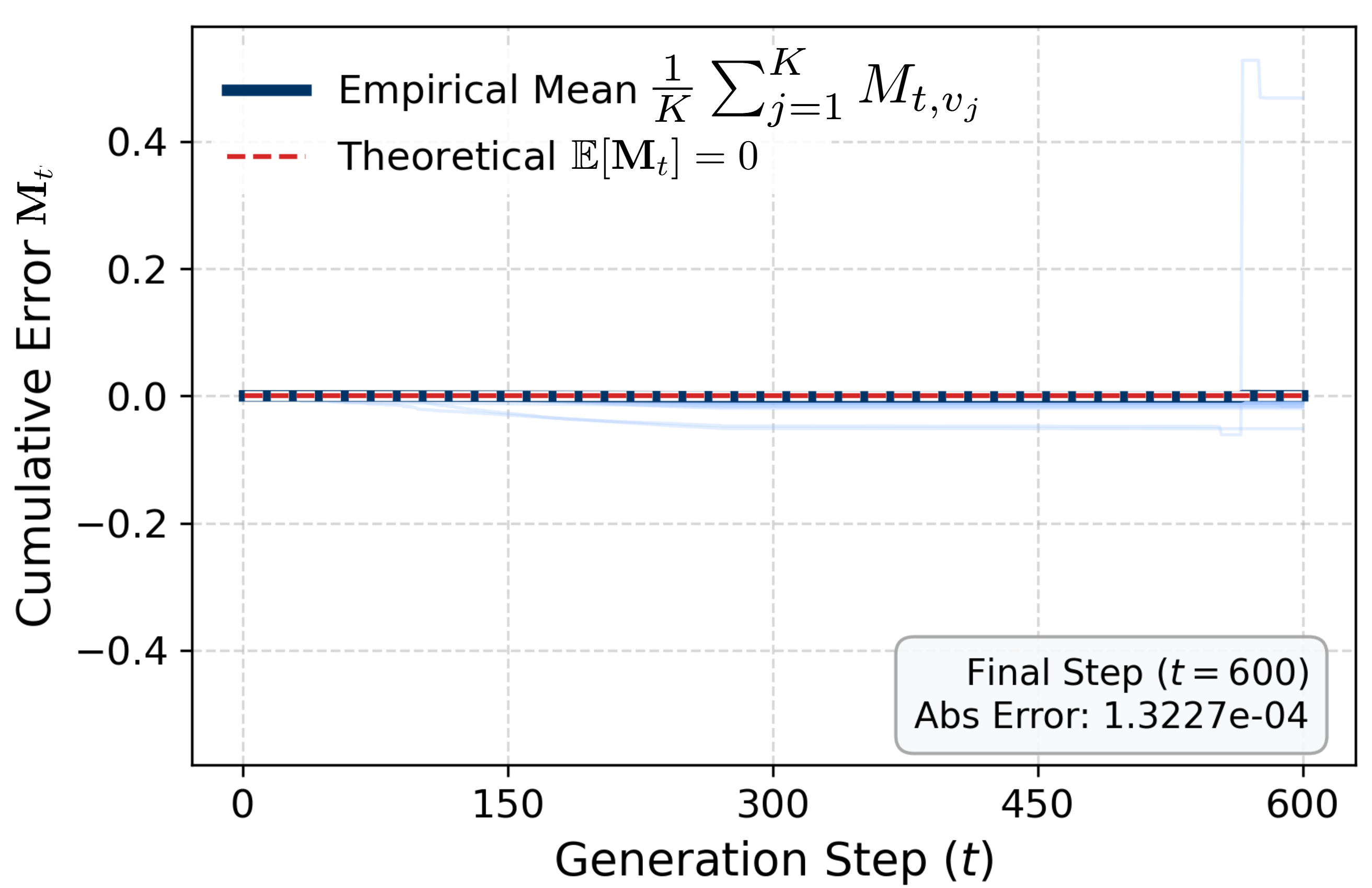}
    \caption{
    Empirical verification of the first-order martingale property on Llama-3.2-3B-Instruct.
    The sampled vocabulary dimensions exhibit random fluctuations,
    while their averaged cumulative error remains around the theoretical expectation
    $\mathbb{E}[\mathbf{M}_t]=0$.
    }
    \label{fig:first_order}
\end{figure}

\subsection{Second-Order Moment Identity Verification}

Beyond the first-order property, Theorem~\ref{thm:raw-easy-calculate}
and equation~\eqref{eq:M-square-formula-easy}
provide a tractable estimation of the second-order moment.

This result allows us to verify the theoretical quadratic variation
using only the probabilities of the actually generated tokens.

To verify this relationship, we conduct Monte Carlo sampling with $K=30$ independent generation sequences using Llama-3.2-3B-Instruct.
For each sequence, we compute the empirical quadratic variation
$\|\mathbf{M}_t\|_2^2$
and compare its Monte Carlo estimation with the theoretical expectation
given by equation~\eqref{eq:M-square-formula-easy}.

Figure~\ref{fig:second_order}
demonstrates that the empirical sequence closely follows the theoretical curve.
The consistency between both quantities verifies the derived quadratic variation identity.

\begin{figure}[htbp]
    \centering
    \includegraphics[width=\linewidth]{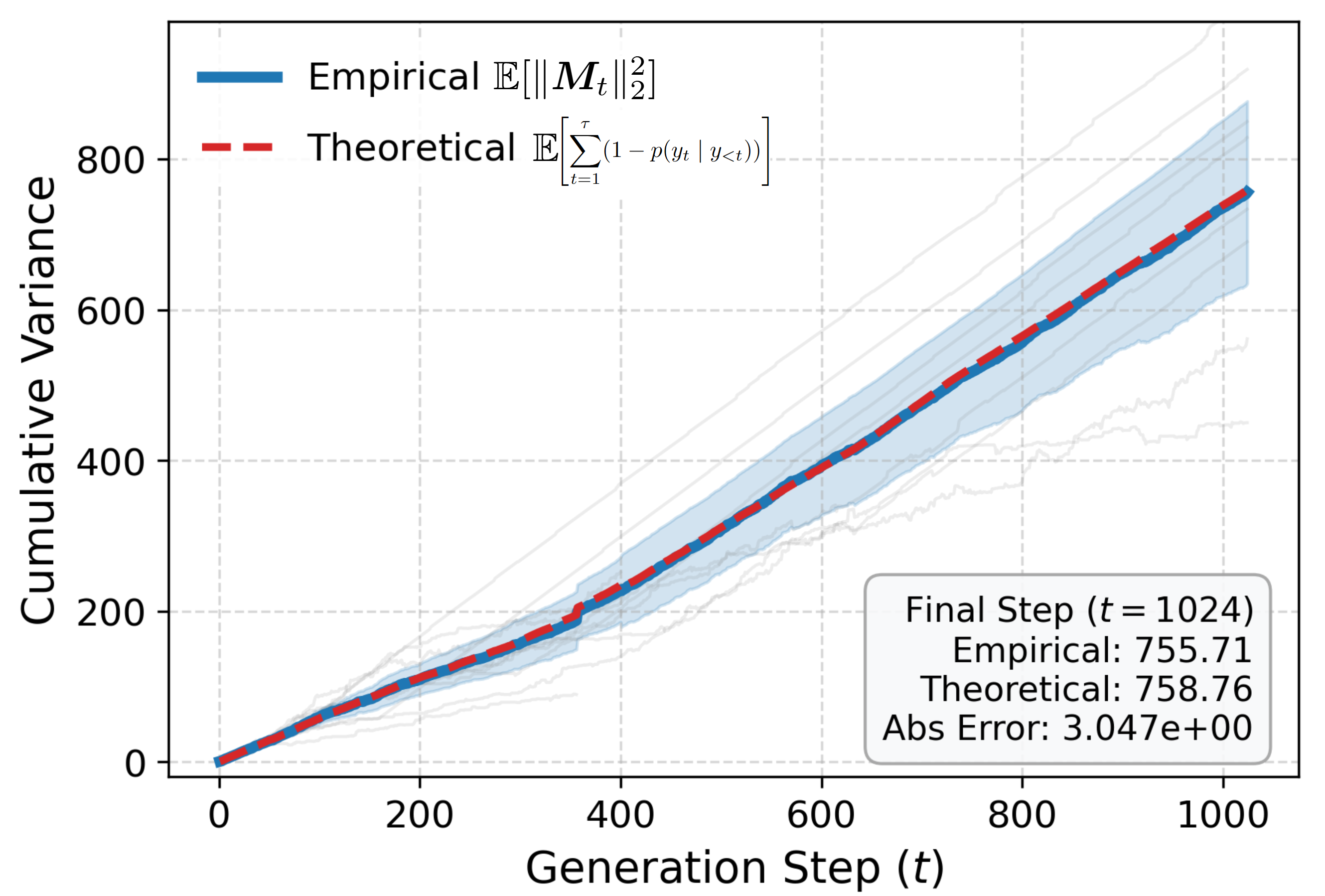}
    \caption{
    Verification of the second-order martingale identity on Llama-3.2-3B-Instruct.
    The empirical second-order moment closely matches the theoretical prediction
    throughout the generation sequence.
    }
    \label{fig:second_order}
\end{figure}

\begin{table}[t]
    \centering
    \begin{tabular*}{\linewidth}{@{\extracolsep{\fill}}cccc}
        \hline
        Model & Empirical & Theoretical & Error\\
        \hline
        Llama-1B &861.60&865.23&3.63\\
        Llama-3B&796.40&781.61&14.79\\
        Llama-8B&708.74&698.66&10.09\\
        Llama-70B&475.45&462.09&13.36\\
        \hline
        Qwen-1.5B&609.81&622.64&12.84\\
        Qwen-3B&434.98&436.80&1.83\\
        Qwen-7B&282.13&283.84&1.72\\
        Qwen-72B&210.35&214.05&3.70\\
        \hline
    \end{tabular*}
    \caption{
    Cross-model verification of the second-order martingale identity.
    The empirical second-order moment consistently agrees with the theoretical quadratic variation across different model scales.
    }
    \label{tab:scaling}
\end{table}

\subsection{Scaling and Generality Analysis}

To ensure these second-order behaviors generalize, we evaluate the Llama-3.x and Qwen2.5 families (1B to 72B parameters). Final-step statistics are summarized in Table~\ref{tab:scaling}, while Figure~\ref{fig:generality} presents complete trajectory comparisons. Across varying model scales, sampling temperatures, and prompt categories, the empirical moments consistently match theoretical predictions. This confirms that our derived identity characterizes an intrinsic property of autoregressive decoding rather than a model-specific artifact. 
Interestingly, under identical prompts, we observe that the averaged quadratic variation varies consistently across different models, hinting at its potential utility as a model-specific probabilistic fingerprint.

\begin{figure}[t]
\centering

\begin{subfigure}{0.48\linewidth}
    \centering
    \includegraphics[width=\linewidth]{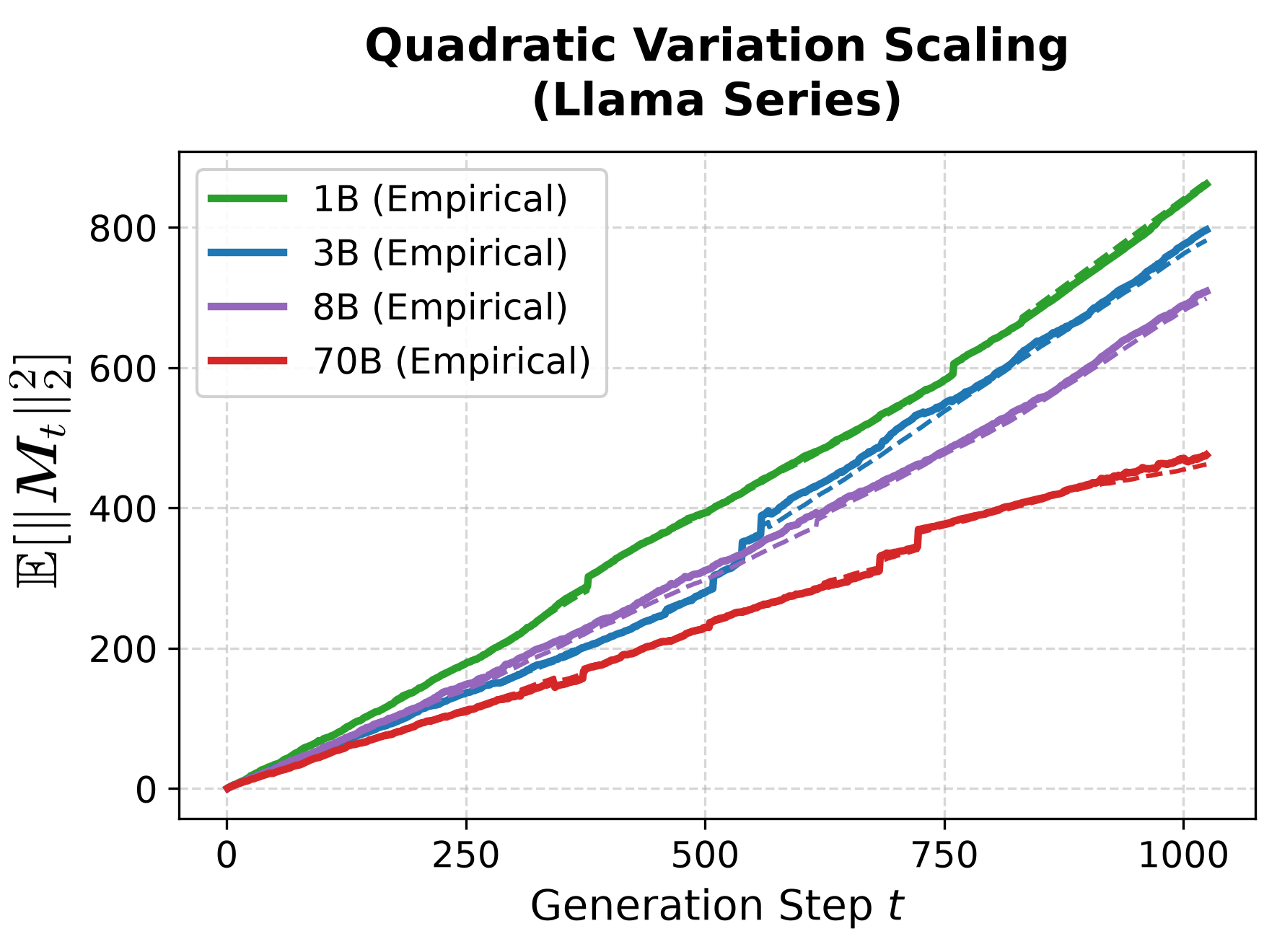}
    \caption{Scaling analysis on Llama models.}
\end{subfigure}
\hfill
\begin{subfigure}{0.48\linewidth}
    \centering
    \includegraphics[width=\linewidth]{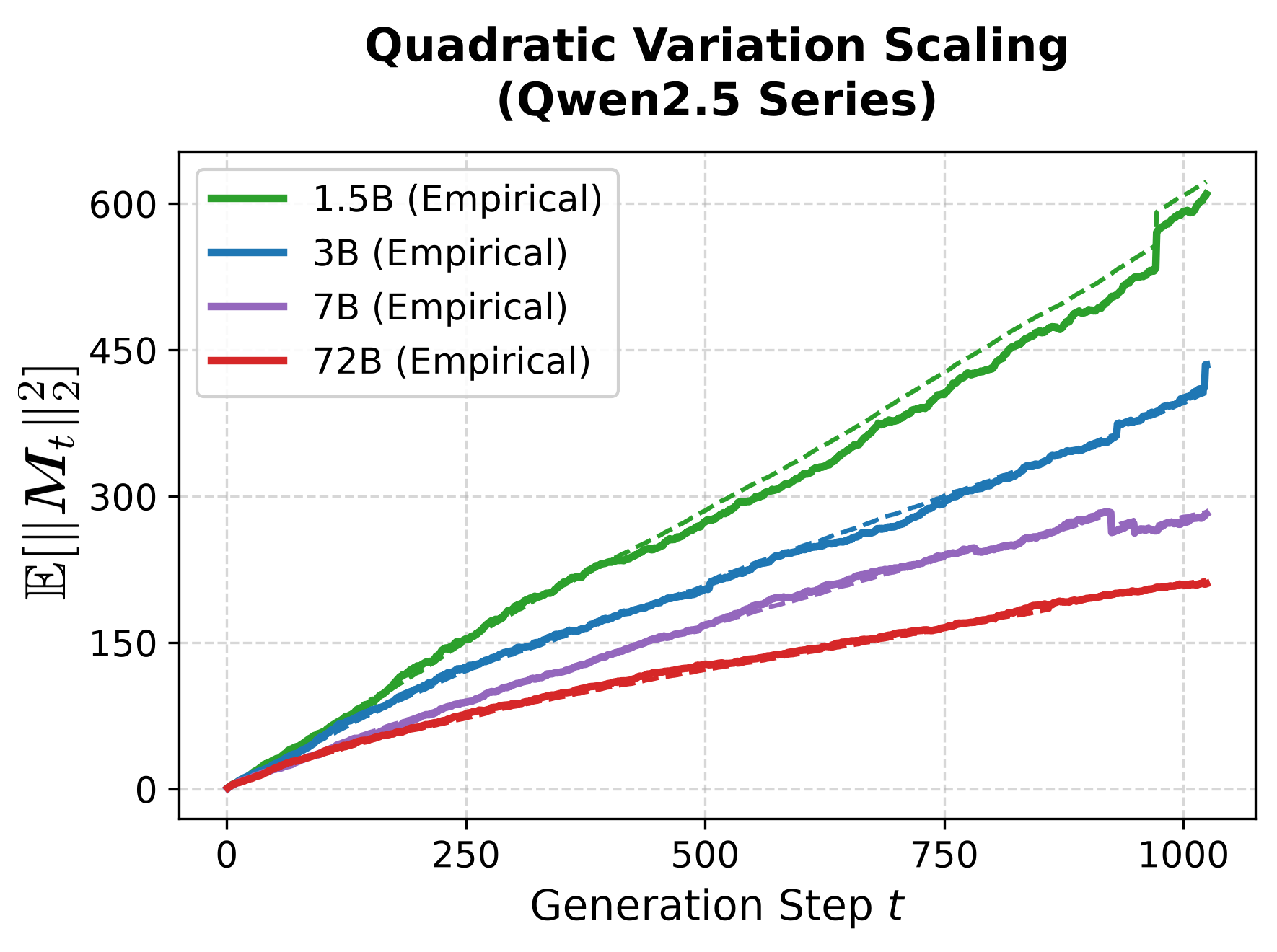}
    \caption{Scaling analysis on Qwen models.}
\end{subfigure}

\vspace{0.3cm}

\begin{subfigure}[t]{0.48\linewidth}
    \centering
    \includegraphics[width=\linewidth]{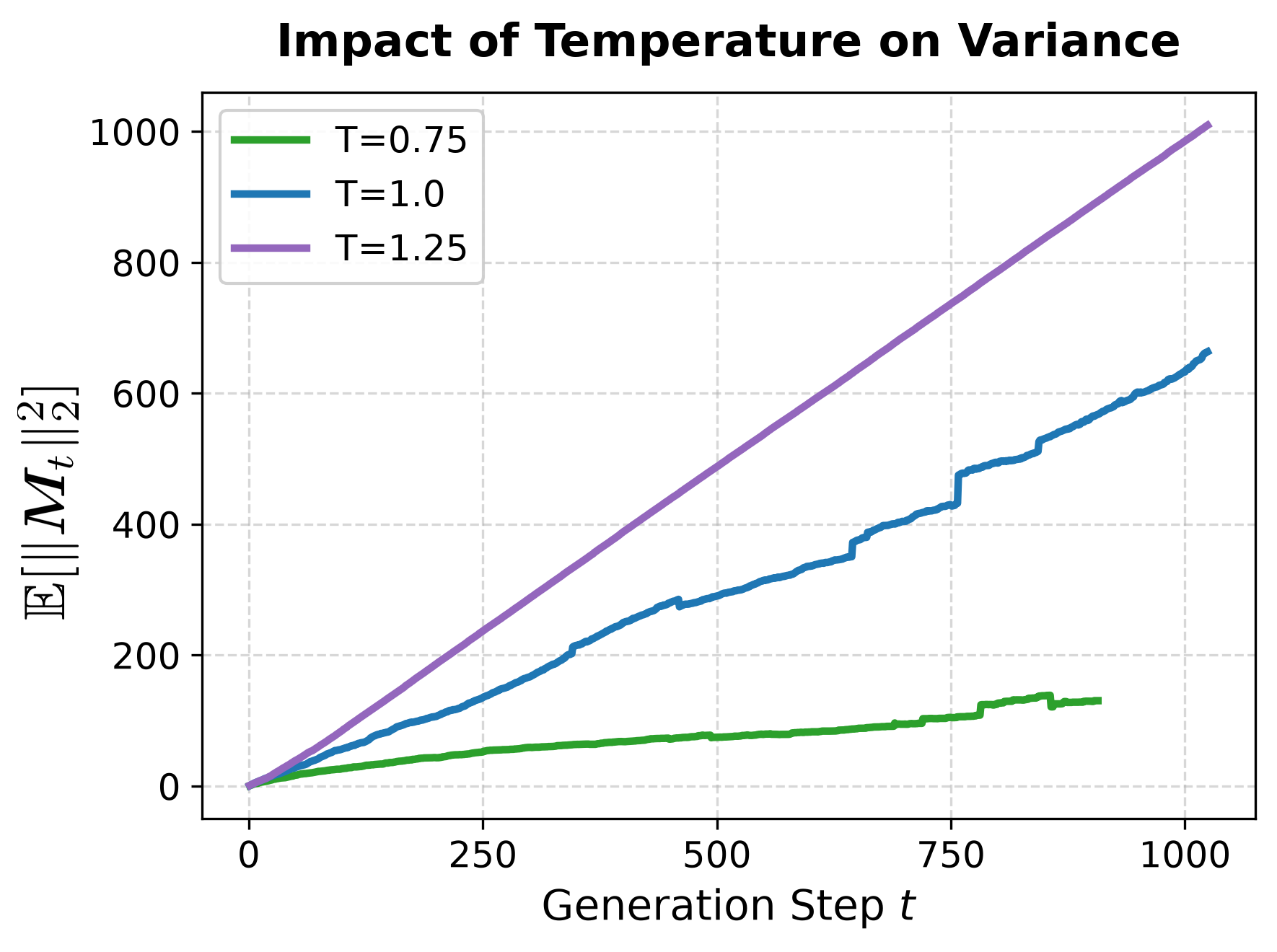}
    \caption{Effect of sampling temperature.}
\end{subfigure}
\hfill
\begin{subfigure}[t]{0.48\linewidth}
    \centering
    \includegraphics[width=\linewidth]{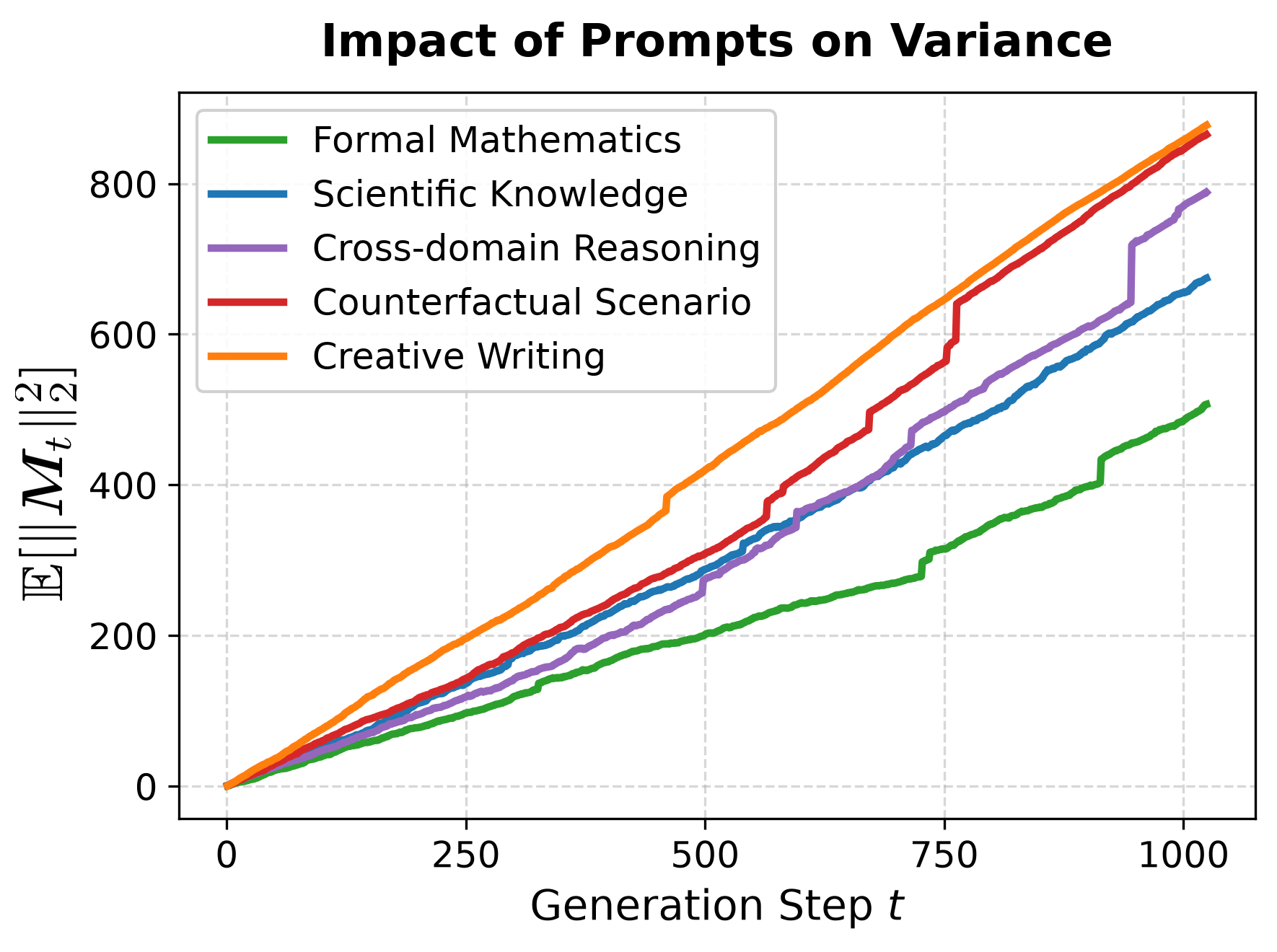}
    \caption{Effect of prompt variations.}
\end{subfigure}

\caption{
Generality analysis of the second-order martingale behavior.
The empirical second-order moments consistently agree with the theoretical
quadratic variation across different model scales, sampling temperatures,
and prompt categories.
}
\label{fig:generality}
\end{figure}

\subsection{Relationship with Existing Uncertainty Measures}

To interpret our derived second-order quantity, we compute its Pearson correlation against standard uncertainty proxies (predictive entropy, maximum probability, and perplexity) over 300 unconstrained trajectories from the RAID benchmark~\cite{dugan2024raid}. 

As Table~\ref{tab:uncertainty_corr} shows, the length-normalized second-order quantity (NSM) is highly correlated with conventional token-level metrics. Crucially, the non-normalized cumulative form (SM) maintains only a moderate correlation. This demonstrates that our proposed martingale metric captures distinct trajectory-level global drift and cumulative stochastic dynamics that are otherwise smoothed out by conventional step-averaged uncertainty measures.

\begin{table}[t]
\centering
\begin{tabular}{lccc}
\hline
 & Entropy & PPL & MaxProb\\
\hline
SM & 0.7203 & 0.7057 & 0.7018\\
NSM & 0.9889 & 0.9482 & 0.9950\\
\hline
\end{tabular}
\caption{
Pearson correlation between the martingale second-order moment (SM), its normalized form (NSM), and existing uncertainty measures.
}
\label{tab:uncertainty_corr}
\end{table}

The results show that the martingale second-order moment shares related information with existing uncertainty measures.
Interestingly, after normalization by sequence length, the second-order quantity becomes highly correlated with token-level uncertainty metrics.
This indicates that the cumulative form preserves additional sequence-level information beyond conventional averaged uncertainty measures.

\section{Conclusion and Discussion}
\label{sec:conclusion}
In this work, we characterize autoregressive LLM decoding as a rigorous martingale process. By defining per-step prediction errors as the difference between sampled indicators and conditional distributions, we prove that the cumulative error forms a zero-mean martingale and derive a tractable, closed-form identity for its predictable quadratic variation. Extensive white-box experiments confirm that empirical trajectories precisely track our theoretical predictions, proving these stochastic dynamics are intrinsic to autoregressive generation. 

While this paper primarily focuses on theoretical characterizations, our framework offers a robust mathematical foundation for interpreting model decoding dynamics. Future work will explore leveraging these martingale properties for trajectory-level uncertainty quantification, model identification, and adaptive decoding strategies.

\bibliography{aaai2027}

\appendix
\allowdisplaybreaks
\setcounter{theorem}{1}

\section*{A{\quad}Proof}
\begin{theorem}
The stochastic process $\{\bm{M}_t\}_{t \geq 1}$ is a martingale with respect to the filtration $\mathbb{F}$.
\end{theorem}

\begin{proof}
    We have
    \begin{equation}
        \begin{aligned}
            &\quad\ \mathbb{E}_{y_t}[\bm{M}_t\mid y_{<t}] \\
            &=\mathbb{E}_{y_t}[\bm{M}_{t-1}\mid y_{<t}] + \mathbb{E}_{y_t}[\bm{e}_t\mid y_{<t}]\\
            &=\bm{M}_{t-1} + \mathbb{E}_{y_t}[\bm{w}_t-\bm{g}_t\mid y_{<t}]\\
            &=\bm{M}_{t-1} + \mathbb{E}_{y_t}[\bm{w}_t\mid y_{<t}] - \bm{g}_t\\
            &=\bm{M}_{t-1} + \bm{g}_t - \bm{g}_t\\
            &=\bm{M}_{t-1}
        \end{aligned}
    \end{equation}
    \begin{equation}
        \begin{gathered}
            \mathbb{E}_{y_{\leq t}}[\|\bm{M}_t\|_2]\leq \mathbb{E}_{y_{\leq t}}\left[\sum_{i=1}^t\|\bm{e}_i\|_2\right]\leq\\
            \mathbb{E}_{y_{\leq t}}\left[\sum_{i=1}^t(\|\bm{w}_i\|_2 + \|\bm{g}_i\|_2)\right] \leq  \left[\sum_{i=1}^t2\right] =  2t < \infty
        \end{gathered}
    \end{equation}
    
    $\bm{g}_t$ and $\bm{w}_t$ is decided given $y_{\leq t}$, and so is $\bm{e}_t$. Thus $\bm{M}_t$ is adapted to filtration $\mathbb{F}$. According to the definition, $\bm{M}_t$ is a martingale.
\end{proof}

\begin{theorem}
    The second-oreder moment of martingale $\{M_t\}_{t\geq 1}$ can be expressed as follows:
    \begin{equation}
        \mathbb{E}_{y_{\leq t}}[\|\bm{M}_t\|_2^2] = \mathbb{E}_{y_{< t}}\left[\sum_{i=1}^t(1 - \|\bm{g}_i\|_2^2)\right]
        \label{eq:M-square-formula-apx}
    \end{equation}
\end{theorem}
\begin{proof} 
    \begin{align}
        &\quad\ \mathbb{E}_{y_{\leq t}}[\|\bm{M}_t\|_2^2]=\mathbb{E}_{y_{< t}}[\mathbb{E}_{y_t}[\|\bm{M}_t\|_2^2\mid y_{<t}]]\label{eq:M-square1-apx}\\
        &\notag\\
        &\begin{aligned}
            &\quad\ \mathbb{E}_{y_t}[\|\bm{M}_t\|_2^2\mid y_{<t}]=\mathbb{E}_{y_t}[\|\bm{e}_t+\bm{M}_{t-1}\|_2^2\mid y_{<t}]\\
            &=\mathbb{E}_{y_t}[\|\bm{e}_t\|_2^2\mid y_{<t}] +\mathbb{E}_{y_t}[2\bm{M}_{t-1}\cdot\bm{e}_t\mid y_{<t}] +\\
            &\quad\ \mathbb{E}_{y_t}[\|\bm{M}_{t-1}\|_2^2\mid y_{<t}]\\
            &=\mathbb{E}_{y_t}[\|\bm{e}_t\|_2^2\mid y_{<t}]+2\bm{M}_{t-1}\cdot\mathbb{E}_{y_t}[\bm{e}_t\mid y_{<t}] + \\
            &\quad\ \|\bm{M}_{t-1}\|_2^2\\
            &=\mathbb{E}_{y_t}[\|\bm{e}_t\|_2^2\mid y_{<t}]+ \|\bm{M}_{t-1}\|_2^2
            \label{eq:M-square2-apx}
        \end{aligned}\\
        &\notag\\
        &\begin{aligned}
            &\quad\ \mathbb{E}_{y_t}[\|\bm{e}_t\|_2^2\mid y_{<t}]=\mathbb{E}_{y_t}[\|\bm{w}_t-\bm{g}_t\|_2^2\mid y_{<t}]\\
            &=\mathbb{E}_{y_t}[\|\bm{w}_t\|_2^2\mid y_{<t}] - \mathbb{E}_{y_t}[2\bm{g}_t\cdot\bm{w}_t\mid y_{<t}] +\\
            &\quad\ \mathbb{E}_{y_t}[\|\bm{g}_t\|_2^2\mid y_{<t}]\\
            &=1 - 2\bm{g}_t\cdot\mathbb{E}_{y_t}[\bm{w}_t\mid y_{<t}] + \|\bm{g}_t\|_2^2\\
            &=1 - 2\bm{g}_t\cdot\bm{g}_t + \|\bm{g}_t\|_2^2\\
            &=1 - \|\bm{g}_t\|_2^2
            \label{eq:M-square3-apx}
        \end{aligned}
    \end{align}
    
    By combining equations \eqref{eq:M-square1-apx}, \eqref{eq:M-square2-apx} and \eqref{eq:M-square3-apx}, we can get
    \begin{equation}
        \begin{gathered}
            \mathbb{E}_{y_{\leq t}}[\|\bm{M}_t\|_2^2]=\mathbb{E}_{y_{< t}}[1 - \|\bm{g}_t\|_2^2 + \|\bm{M}_{t-1}\|_2^2] =\\
            \mathbb{E}_{y_{< t}}[1 - \|\bm{g}_t\|_2^2] + \mathbb{E}_{y_{\leq t-1}}[\|\bm{M}_{t-1}\|_2^2]
        \end{gathered}
    \end{equation}
    
    Then the following recurrence relation holds:
    \begin{equation}
        \mathbb{E}_{y_{\leq t}}[\|\bm{M}_t\|_2^2] - \mathbb{E}_{y_{\leq t-1}}[\|\bm{M}_{t-1}\|_2^2] = \mathbb{E}_{y_{< t}}[1 - \|\bm{g}_t\|_2^2] 
    \end{equation}
    
    Thus
    \begin{equation}
        \begin{aligned}
            &\quad\ \mathbb{E}_{y_{\leq t}}[\|\bm{M}_t\|_2^2] - \mathbb{E}_{y_{\leq 1}}[\|\bm{M}_1\|_2^2]\\ 
            &=\mathbb{E}_{y_{\leq t}}[\|\bm{M}_t\|_2^2] - \mathbb{E}_{y_1}[\|\bm{e}_1\|_2^2\mid y_{<1}] \\
            &=\mathbb{E}_{y_{\leq t}}[\|\bm{M}_t\|_2^2] - (1 - \|\bm{g}_1\|_2^2)\\
            &=\mathbb{E}_{y_{< t}}\left[\sum_{i=2}^t(1 - \|\bm{g}_i\|_2^2)\right]
        \end{aligned}
    \end{equation}
    
    So
    \begin{equation}
        \mathbb{E}_{y_{\leq t}}[\|\bm{M}_t\|_2^2] = \mathbb{E}_{y_{<t}}\left[\sum_{i=1}^t(1 - \|\bm{g}_i\|_2^2)\right]
    \end{equation}    
\end{proof}

\begin{theorem}
    The second order moment of martingale $\{\bm{M}_t\}_{t\geq 1}$ with stopping time can be expressed as follows:
    \begin{equation}
        \mathbb{E}_{y_{\leq l_{\mathrm{max}}}}[\|\bm{M}_{\tau}\|_2^2]=\mathbb{E}_{y_{\leq l_{\mathrm{max}}}}\left[\sum_{t=1}^{\tau}(1-\|\bm{g}_t\|_2^2)\right]
        \label{eq:M-square-formula-terminate-apx}
    \end{equation}
\end{theorem}
\begin{proof}

    \begin{gather}
            \mathbb{E}_{y_{l_{\mathrm{max}}}}[\|\bm{M}_{\tau}\|_2^2\mid y_{<l_{\mathrm{max}}}] \notag\\
            =\mathbb{E}_{y_{l_{\mathrm{max}}}}\left[\sum_{t=1}^{l_{\mathrm{max}}}\mathbb{I}_{\{\tau=t\}}\|\bm{M}_t\|_2^2\mid y_{<l_{\mathrm{max}}}\right]\notag\\
            =\mathbb{E}_{y_{l_{\mathrm{max}}}}[\mathbb{I}_{\{\tau=l_{\mathrm{max}}\}}\|\bm{M}_{l_{\mathrm{max}}}\|_2^2\mid y_{<l_{\mathrm{max}}}] +\notag\\
            \sum_{t=1}^{l_{\mathrm{max}}-1}\mathbb{I}_{\{\tau=t\}}\|\bm{M}_t\|_2^2\notag\\
            \begin{gathered}
                =\mathbb{I}_{\{\tau=l_{\mathrm{max}}\}}\mathbb{E}_{y_{l_{\mathrm{max}}}}[\|\bm{M}_{l_{\mathrm{max}}}\|_2^2\mid y_{<l_{\mathrm{max}}}] +\\
                \sum_{t=1}^{l_{\mathrm{max}}-1}\mathbb{I}_{\{\tau=t\}}\|\bm{M}_t\|_2^2
            \end{gathered}\\
            =\mathbb{I}_{\{\tau=l_{\mathrm{max}}\}}(1-\|\bm{g}_{l_{\mathrm{max}}}\|_2^2 +
            \|\bm{M}_{l_{\mathrm{max}}-1}\|_2^2) +\notag\\ \sum_{t=1}^{l_{\mathrm{max}}-1}\mathbb{I}_{\{\tau=t\}}\|\bm{M}_t\|_2^2\notag\\
            =\mathbb{I}_{\{\tau=l_{\mathrm{max}}\}}(1-\|\bm{g}_{l_{\mathrm{max}}}\|_2^2)+\notag\\
            \mathbb{I}_{\{\tau\geq l_{\mathrm{max}}-1\}}\|\bm{M}_{l_{\mathrm{max}}-1}\|_2^2 +
            \sum_{t=1}^{l_{\mathrm{max}}-2}\mathbb{I}_{\{\tau=t\}}\|\bm{M}_t\|_2^2
            \notag
    \end{gather}
    Taking conditional expectation on both sides with respect to $y_{<l_{\mathrm{max}}-1}$:
    \begin{equation}
        \begin{gathered}
            \mathbb{E}_{y_{l_{\mathrm{max}}-1}y_{l_{\mathrm{max}}}}[\|\bm{M}_{\tau}\|_2^2\mid y_{<l_{\mathrm{max}}-1}]\\
            =\mathbb{E}_{y_{l_{\mathrm{max}}-1}}[\mathbb{E}_{y_{l_{\mathrm{max}}}}[\|\bm{M}_{\tau}\|_2^2\mid y_{<l_{\mathrm{max}}}]\mid y_{<l_{\mathrm{max}}-1}] \\
            =  \mathbb{E}_{y_{l_{\mathrm{max}}-1}}[\mathbb{I}_{\{\tau=l_{\mathrm{max}}\}}(1-\|\bm{g}_{l_{\mathrm{max}}}\|_2^2)\mid y_{<l_{\mathrm{max}}-1}] + \\
            \mathbb{I}_{\{\tau\geq l_{\mathrm{max}}-1\}} \mathbb{E}_{y_{l_{\mathrm{max}}-1}}[\|\bm{M}_{l_{\mathrm{max}}-1}\|_2^2\mid y_{<l_{\mathrm{max}}-1}] + \\
            \sum_{t=1}^{l_{\mathrm{max}}-2}\mathbb{I}_{\{\tau=t\}}\|\bm{M}_t\|_2^2\\
            = \mathbb{E}_{y_{l_{\mathrm{max}}-1}}[\mathbb{I}_{\{\tau=l_{\mathrm{max}}\}}(1-\|\bm{g}_{l_{\mathrm{max}}}\|_2^2)\mid y_{<l_{\mathrm{max}}-1}] +\\
            \mathbb{I}_{\{\tau\geq l_{\mathrm{max}}-1\}}(1-\|\bm{g}_{l_{\mathrm{max}}-1}\|_2^2 + \|\bm{M}_{l_{\mathrm{max}}-2}\|_2^2) + \\
            \sum_{t=1}^{l_{\mathrm{max}}-2}\mathbb{I}_{\{\tau=t\}}\|\bm{M}_t\|_2^2\\
            =\mathbb{E}_{y_{l_{\mathrm{max}}-1}}[\mathbb{I}_{\{\tau\geq l_{\mathrm{max}}\}}(1-\|\bm{g}_{l_{\mathrm{max}}}\|_2^2)\mid y_{<l_{\mathrm{max}}-1}] + \\
            \mathbb{I}_{\{\tau\geq l_{\mathrm{max}}-1\}}(1-\|\bm{g}_{l_{\mathrm{max}}-1}\|_2^2)+\\
            \mathbb{I}_{\{\tau\geq l_{\mathrm{max}}-2\}}\|\bm{M}_{l_{\mathrm{max}}-2}\|_2^2 + 
            \sum_{t=1}^{l_{\mathrm{max}}-3}\mathbb{I}_{\{\tau=t\}}\|\bm{M}_t\|_2^2
        \end{gathered}
    \end{equation}
    Then, taking conditional expectations successively with respect to $y_{<i}\ (i=l_{\mathrm{max}}-2,\cdots,1)$, we have
    \begin{equation}
        \begin{aligned}
            &\quad\ \mathbb{E}_{y_{\leq l_{\mathrm{max}}}}[\|\bm{M}_{\tau}\|_2^2]\\
            &= \mathbb{E}_{y_{\leq l_{\mathrm{max}}}}\left[\sum_{t=1}^{l_{\mathrm{max}}}\mathbb{I}_{\{\tau\geq t\}}(1-\|\bm{g}_t\|_2^2)\right]\\
            &=\mathbb{E}_{y_{\leq l_{\mathrm{max}}}}\left[\sum_{t=1}^{\tau}(1-\|\bm{g}_t\|_2^2)\right]
        \end{aligned}
    \end{equation}
\end{proof}

\begin{theorem}
    The second-oreder moment of martingale $\{M_t\}_{t\geq 1}$ can be expressed as follows:
    \begin{equation}
        \mathbb{E}_{y_{\leq t}}[\|\bm{M}_t\|_2^2] =\mathbb{E}_{y_{\leq t}}\left[\sum_{i=1}^t(1-p(y_i\mid y_{<i}))\right]
    \end{equation}
\end{theorem}
\begin{proof}
    \begin{align}
        &\quad\ \mathbb{E}_{y_{\leq t}}[\|\bm{M}_t\|_2^2]\notag \\
        &=\mathbb{E}_{y_{< t}}\left[\sum_{i=1}^t(1 - \|\bm{g}_i\|_2^2)\right]\notag\\
        &=\sum_{i=1}^t\mathbb{E}_{y_{< t}}[1 - \|\bm{g}_i\|_2^2]\notag\\
        &=\sum_{i=1}^t\mathbb{E}_{y_{< i}}[1 - \|\bm{g}_i\|_2^2]\notag\\
        &=\sum_{i=1}^t\sum_{y_{<i}\in\mathcal{V}^{i-1}}p(y_{< i})(1 - \|\bm{g}_i\|_2^2)\notag\\ 
        &=\sum_{i=1}^t\sum_{y_{<i}\in\mathcal{V}^{i-1}}p(y_{< i})(1 - \sum_{y_i\in\mathcal{V}}p^2(y_i\mid y_{<i}))\notag\\
        &=\sum_{i=1}^t\sum_{y_{<i}\in\mathcal{V}^{i-1}}(p(y_{< i})- \sum_{y_i\in\mathcal{V}}p(y_{< i})p^2(y_i\mid y_{<i}))\\
        &=\sum_{i=1}^t\sum_{y_{<i}\in\mathcal{V}^{i-1}}(p(y_{< i})- \sum_{y_i\in\mathcal{V}}p(y_{\leq i})p(y_i\mid y_{<i}))\notag\\
        &=\sum_{i=1}^t\sum_{y_{\leq i}\in\mathcal{V}^{i}}p(y_{\leq i})(1-p(y_i\mid y_{<i}))\notag\\
        &=\sum_{i=1}^t\mathbb{E}_{y_{\leq i}}[1-p(y_i\mid y_{<i})]\notag\\
        &=\sum_{i=1}^t\mathbb{E}_{y_{\leq t}}[1-p(y_i\mid y_{<i})]\notag\\
        &= \mathbb{E}_{y_{\leq t}}\left[\sum_{i=1}^t(1-p(y_i\mid y_{<i}))\right]\notag
    \end{align}
\end{proof}

\begin{theorem}
    The second order moment of martingale $\{\bm{M}_t\}_{t\geq 1}$ with stopping time can be expressed as follows:
    \begin{equation}
        \mathbb{E}_{y_{\leq l_{\mathrm{max}}}}[\|\bm{M}_{\tau}\|_2^2]
        =\mathbb{E}_{y_{\leq l_{\mathrm{max}}}}\left[\sum_{t=1}^{\tau}(1-p(y_t\mid y_{<t}))\right]
    \end{equation}
\end{theorem}
\begin{proof}
    \begin{align}
        &\quad\ \mathbb{E}_{y_{\leq l_{\mathrm{max}}}}[\|\bm{M}_{\tau}\|_2^2]=\mathbb{E}_{y_{\leq l_{\mathrm{max}}}}\left[\sum_{t=1}^{l_{\mathrm{max}}}\mathbb{I}_{\{\tau\geq t\}}(1-\|\bm{g}_t\|_2^2)\right]\notag\\
        &=\sum_{t=1}^{l_{\mathrm{max}}}\mathbb{E}_{y_{\leq l_{\mathrm{max}}}}[\mathbb{I}_{\{\tau\geq t\}}(1-\|\bm{g}_t\|_2^2)]\notag\\
        &=\sum_{t=1}^{l_{\mathrm{max}}}\mathbb{E}_{y_{<t}}[\mathbb{I}_{\{\tau\geq t\}}(1-\|\bm{g}_t\|_2^2)]\notag\\
        &=\sum_{t=1}^{l_{\mathrm{max}}}\sum_{y_{<t}\in\mathcal{V}^{t-1}}p(y_{<t})\mathbb{I}_{\{\tau\geq t\}}(1-\sum_{y_t\in\mathcal{V}}p^2(y_t\mid y_{<t}))\notag\\
        &=\sum_{t=1}^{l_{\mathrm{max}}}\sum_{y_{<t}\in\mathcal{V}^{t-1}}\mathbb{I}_{\{\tau\geq t\}}(p(y_{<t})-\sum_{y_t\in\mathcal{V}}p(y_{<t})p^2(y_t\mid y_{<t}))\notag\\
        &=\sum_{t=1}^{l_{\mathrm{max}}}\sum_{y_{<t}\in\mathcal{V}^{t-1}}\mathbb{I}_{\{\tau\geq t\}}(p(y_{<t})-\sum_{y_t\in\mathcal{V}}p(y_{\leq t})p(y_t\mid y_{<t}))\notag\\
        &=\sum_{t=1}^{l_{\mathrm{max}}}\sum_{y_{\leq t}\in\mathcal{V}^t}p(y_{\leq t})\mathbb{I}_{\{\tau\geq t\}}(1-p(y_t\mid y_{<t}))\\
        &= \sum_{t=1}^{l_{\mathrm{max}}}\mathbb{E}_{y_{\leq t}}[\mathbb{I}_{\{\tau\geq t\}}(1-p(y_t\mid y_{<t}))]\notag\\
        &=\sum_{t=1}^{l_{\mathrm{max}}}\mathbb{E}_{y_{\leq l_{\mathrm{max}}}}[\mathbb{I}_{\{\tau\geq t\}}(1-p(y_t\mid y_{<t}))]\notag\\
        &=\mathbb{E}_{y_{\leq l_{\mathrm{max}}}}\left[\sum_{t=1}^{l_{\mathrm{max}}}\mathbb{I}_{\{\tau\geq t\}}(1-p(y_t\mid y_{<t}))\right]\notag\\
        &=\mathbb{E}_{y_{\leq l_{\mathrm{max}}}}\left[\sum_{t=1}^{\tau}(1-p(y_t\mid y_{<t}))\right]\notag
    \end{align}
\end{proof}

\begin{theorem}
    The second-oreder moment of martingale (with stopping time) can be expressed as follows:
    \begin{gather}
        \mathbb{E}_{y_{\leq t}}[\|\bm{M}_t\|_2^2]=\mathbb{E}_{y_{\leq t}}\left[\sum_{i=1}^t \bm{e}_i^2\right]\label{eq:M-error-apx}\\
        \mathbb{E}_{y_{\leq l_{\mathrm{max}}}}[\|\bm{M}_{\tau}\|_2^2]=\mathbb{E}_{y_{\leq l_{\mathrm{max}}}}\left[\sum_{t=1}^{\tau} \bm{e}_t^2\right]\label{eq:M-error-terminate-apx}
    \end{gather}
\end{theorem}
\begin{proof}
    \begin{equation}
        \begin{aligned}
            &\quad\ \mathbb{E}_{y_{\leq l_{\mathrm{max}}}}[\bm{e}_i^2]=\mathbb{E}_{y_{\leq i}}[\bm{e}_i^2]\\
            &=\mathbb{E}_{y_{\leq i}}[(\bm{w}_i-\bm{g}_i)^2]\\
            &=\mathbb{E}_{y_{\leq i}}[\bm{w}_i^2-2\bm{w}_i\cdot\bm{g}_i+\bm{g}_i^2]\\
            &=\mathbb{E}_{y_{\leq i}}[\bm{w}_i^2]-2\mathbb{E}_{y_{\leq i}}[\bm{w}_i\cdot\bm{g}_i]+\mathbb{E}_{y_{\leq i}}[\bm{g}_i^2]\\
            &=1-2\mathbb{E}_{y_{< i}}[\mathbb{E}_{y_i}[\bm{w}_i\cdot\bm{g}_i\mid y_{<i}]]+\mathbb{E}_{y_{< i}}[\bm{g}_i^2]\\
            &=1-2\mathbb{E}_{y_{< i}}[\bm{g}_i^2]+\mathbb{E}_{y_{< i}}[\bm{g}_i^2]\\
            &=\mathbb{E}_{y_{< i}}[1-\bm{g}_i^2]
            \label{eq:error-square-apx}
        \end{aligned}
    \end{equation}
    \begin{equation}
        \begin{aligned}
            &\quad\ \mathbb{E}_{y_{\leq l_{\mathrm{max}}}}[\mathbb{I}_{\{\tau\geq t\}}\bm{e}_t^2]=\mathbb{E}_{y_{\leq t}}[\mathbb{I}_{\{\tau\geq t\}}\bm{e}_t^2]\\
            &=\mathbb{E}_{y_{\leq t}}[\mathbb{I}_{\{\tau\geq t\}}(\bm{w}_t-\bm{g}_t)^2]\\
            &=\mathbb{E}_{y_{\leq t}}[\mathbb{I}_{\{\tau\geq t\}}(\bm{w}_t^2-2\bm{w}_t\cdot\bm{g}_t+\bm{g}_t^2)]\\
            &=\mathbb{E}_{y_{< t}}[\mathbb{E}_{y_t}[\mathbb{I}_{\{\tau\geq t\}}(1-2\bm{w}_t\cdot\bm{g}_t+\bm{g}_t^2)\mid y_{<t}]]\\
            &=\mathbb{E}_{y_{< t}}[\mathbb{I}_{\{\tau\geq t\}}\mathbb{E}_{y_t}[1-2\bm{w}_t\cdot\bm{g}_t+\bm{g}_t^2\mid y_{<t}]]\\
            &=\mathbb{E}_{y_{< t}}[\mathbb{I}_{\{\tau\geq t\}}(1-\bm{g}_t^2)]
            \label{eq:error-square-terminate-apx}
        \end{aligned}
    \end{equation}
    By combining equations \eqref{eq:error-square-apx} and \eqref{eq:error-square-terminate-apx} with equations \eqref{eq:M-square-formula-apx} and \eqref{eq:M-square-formula-terminate-apx}, respectively, we obtain equations \eqref{eq:M-error-apx} and \eqref{eq:M-error-terminate-apx}.
\end{proof}

\begin{theorem}
    For the second order moments $\mathbb{E}_{y_{\leq t}}[\|\bm{M}_t\|_2^2]$ and $\mathbb{E}_{y_{\leq l_{\mathrm{max}}}}[\|\bm{M}_T\|_2^2]$, the following relation holds:
    \begin{gather}
         \sum_{i=1}^t\mathrm{Var}(\bm{w}_i)=\mathbb{E}_{y_{\leq t}}[\|\bm{M}_t\|_2^2]+\sum_{i=1}^t\mathrm{Var}(\bm{g}_i)\label{eq:M-decompose-apx}\\
         \begin{gathered}
             \mathbb{E}_{y_{\leq l_{\mathrm{max}}}}\left[\sum_{t=1}^{\tau}\mathrm{Var}(\bm{w}_t)\right]\\
             =\mathbb{E}_{y_{\leq l_\mathrm{max}}}[\|\bm{M}_{\tau}\|_2^2]+\mathbb{E}_{y_{\leq l_{\mathrm{max}}}}\left[\sum_{t=1}^{\tau}\mathrm{Var}(\bm{g}_t)\right]
         \end{gathered}
         \label{eq:M-decompose-terminate-apx}
    \end{gather}

    Then
    \begin{gather}
        \sum_{i=1}^t\mathrm{Var}(\bm{w}_i)\geq \mathbb{E}_{y_{\leq t}}[\|\bm{M}_t\|_2^2]\label{eq:M-lower-bound-apx}\\
         \mathbb{E}_{y_{\leq l_{\mathrm{max}}}}\left[\sum_{t=1}^{\tau}\mathrm{Var}(\bm{w}_t)\right]\geq\mathbb{E}_{y_{\leq l_\mathrm{max}}}[\|\bm{M}_{\tau}\|_2^2]\label{eq:M-lower-bound-terminate-apx}
    \end{gather}
\end{theorem}
\begin{proof}
    By the law of total variance, we have
    \begin{equation}
        \begin{aligned}
            &\quad\ \mathrm{Var}(\bm{w}_i)\\
            &=\mathbb{E}_{y_{< i}}[\mathrm{Var}(\bm{w}_i\mid y_{< i})]+\mathrm{Var}(\mathbb{E}_{y_i}[\bm{w}_i\mid y_{< i}])\\
            &=\mathbb{E}_{y_{< i}}[\mathbb{E}_{y_i}[(\bm{w}_i-\mathbb{E}_{y_i}[\bm{w}_i\mid y_{< i}])^2\mid y_{< i}]]+\\
            &\quad\ \mathrm{Var}(\mathbb{E}_{y_i}[\bm{w}_i\mid y_{< i}])\\
            &=\mathbb{E}_{y_{< i}}[\mathbb{E}_{y_i}[(\bm{w}_i-\bm{g}_i)^2\mid y_{< i}]]+\mathrm{Var}(\bm{g}_i)\\
            &=\mathbb{E}_{y_{\leq i}}[\bm{e}_i^2]+\mathrm{Var}(\bm{g}_i)
        \end{aligned}
        \label{eq:var-decompose-apx}
    \end{equation}
    \begin{align}
        &\quad\ \mathrm{Var}(\mathbb{I}_{\{\tau\geq t\}}\bm{w}_t)=\mathbb{E}_{y_{< t}}[\mathrm{Var}(\mathbb{I}_{\{\tau\geq t\}}\bm{w}_t\mid y_{< t})]+\notag\\
        &\quad\ \mathrm{Var}(\mathbb{E}_{y_t}[\mathbb{I}_{\{\tau\geq t\}}\bm{w}_t\mid y_{< t}])\notag\\
        &=\mathbb{E}_{y_{< t}}[\mathbb{E}_{y_t}[(\mathbb{I}_{\{\tau\geq t\}}\bm{w}_t-\mathbb{E}_{y_t}[\mathbb{I}_{\{\tau\geq t\}}\bm{w}_t\mid y_{< t}])^2\mid y_{< t}]]+\notag\\
        &\quad\ \mathrm{Var}(\mathbb{E}_{y_t}[\mathbb{I}_{\{\tau\geq t\}}\bm{w}_t\mid y_{< t}])\label{eq:var-decompose-terminate-apx}\\
        &=\mathbb{E}_{y_{< t}}[\mathbb{E}_{y_t}[(\mathbb{I}_{\{\tau\geq t\}}\bm{w}_t-\mathbb{I}_{\{\tau\geq t\}}\mathbb{E}_{y_t}[\bm{w}_t\mid y_{< t}])^2\mid y_{< t}]]+\notag\\
        &\quad\ \mathrm{Var}(\mathbb{I}_{\{\tau\geq t\}}\mathbb{E}_{y_t}[\bm{w}_t\mid y_{< t}])\notag\\
        &=\mathbb{E}_{y_{< t}}[\mathbb{E}_{y_t}[\mathbb{I}_{\{\tau\geq t\}}(\bm{w}_t-\bm{g}_t)^2\mid y_{< t}]]+\mathrm{Var}(\mathbb{I}_{\{\tau\geq t\}}\bm{g}_t)\notag\\
        &=\mathbb{E}_{y_{\leq t}}[\mathbb{I}_{\{\tau\geq t\}}\bm{e}_t^2]+\mathrm{Var}(\mathbb{I}_{\{\tau\geq t\}}\bm{g}_t)\notag
    \end{align}

    Combining equations \eqref{eq:var-decompose-apx} and \eqref{eq:var-decompose-terminate-apx} with equations \eqref{eq:M-error-apx} and \eqref{eq:M-error-terminate-apx}, respectively, yields equations \eqref{eq:M-decompose-apx} and \eqref{eq:M-decompose-terminate-apx}.

    Noting that 
    \begin{gather}
       \sum_{i=1}^t\mathrm{Var}(\bm{g}_i)\geq 0\\
       \mathbb{E}_{y_{\leq l_{\mathrm{max}}}}\left[\sum_{t=1}^{\tau}\mathrm{Var}(\bm{g}_t)\right]\geq 0
    \end{gather}
    from equations \eqref{eq:M-decompose-apx} and \eqref{eq:M-decompose-terminate-apx} we obtain equations \eqref{eq:M-lower-bound-apx} and \eqref{eq:M-lower-bound-terminate-apx}.
\end{proof}

\begin{theorem}
    The following equivalences hold:
    \begin{gather}
        \sum_{i=1}^t\mathrm{Var}(\bm{w}_i)=0\iff \mathbb{E}_{y_{\leq t}}[\|\bm{M}_t\|_2^2]=0\label{eq:M-zero-error-apx}\\
        \begin{gathered}
           \mathbb{E}_{y_{\leq l_{\mathrm{max}}}}\left[\sum_{t=1}^{\tau}\mathrm{Var}(\bm{w}_t)\right]=0\iff \\
           \mathbb{E}_{y_{\leq l_\mathrm{max}}}[\|\bm{M}_{\tau}\|_2^2]=0 
        \end{gathered}
        \label{eq:M-zero-error-terminate-apx}
    \end{gather}
\end{theorem}
\begin{proof}
    From equations \eqref{eq:M-lower-bound-apx} and \eqref{eq:M-lower-bound-terminate-apx}, the left-to-right implications of equations \eqref{eq:M-zero-error-apx} and \eqref{eq:M-zero-error-terminate-apx} are obvious. We first consider the right-to-left implication of equation~\eqref{eq:M-zero-error-apx}.

    From equation~\eqref{eq:M-error-apx}, $\mathbb{E}_{y_{\leq t}}[\|\bm{M}_t\|_2^2]=0$ is equivalent to $\forall 1\leq i\leq t:\mathbb{E}_{y_{\leq t}}[\bm{e}_i^2]=0$. This implies that $\forall 1\leq i\leq t:(\bm{w}_i=\bm{g}_i\ \mathrm{a.e.})$, meaning that for output sequence $y_1y_2\dots,y_t$ with nonzero probability, $\bm{g}_i$ is a one-hot vector. That is, there is only one output sequence with probability 1, and all other output sequence have probability 0. This means that the output at each token is fixed, and $\sum_{i=1}^t\mathrm{Var}(\bm{w}_i)=0$.

    For equation~\eqref{eq:M-zero-error-terminate-apx}, we consider those $t$ with $P(\tau=t)> 0$; the rest of the proof is similar to that of equation~\eqref{eq:M-zero-error-apx}.
\end{proof}

\begin{theorem}
    The following equivalences hold:
    \begin{gather}
        \begin{gathered}
            \sum_{i=1}^t\mathrm{Var}(\bm{w}_i)=t(1-\frac{1}{|\mathcal{V}|})\iff\\
            \forall 1\leq i\leq t:\mathbb{E}_{y_{< i}}[\bm{g}_i]=\frac{1}{|\mathcal{V}|}\mathbf{1}\\
        \end{gathered}\\
        \begin{gathered}
            \mathbb{E}_{y_{\leq t}}[\|\bm{M}_t\|_2^2]=t(1-\frac{1}{|\mathcal{V}|})\iff \\
            \forall 1\leq i\leq t:(\bm{g}_i=\frac{1}{|\mathcal{V}|}\mathbf{1}\ \mathrm{a.e.})
        \end{gathered}
    \end{gather}
    where $\mathbf{1}$ is the all-one vector.
    
    The following implication hold:
    \begin{equation}
        \begin{gathered}
            \mathbb{E}_{y_{\leq t}}[\|\bm{M}_t\|_2^2]=t(1-\frac{1}{|\mathcal{V}|})\Rightarrow\\
            \sum_{i=1}^t\mathrm{Var}(\bm{w}_i)=t(1-\frac{1}{|\mathcal{V}|})
        \end{gathered}
        \label{eq:entail-for-M-mi-var-max-apx}
    \end{equation}
\end{theorem}
\begin{proof}
    \begin{equation}
        \begin{aligned}
            &\quad\ \sum_{i=1}^t\mathrm{Var}(\bm{w}_i)\\
            &=\sum_{i=1}^t(\mathbb{E}_{y_{\leq i}}[\bm{w}_i^2]-\mathbb{E}_{y_{\leq i}}^2[\bm{w}_i])\\
            &=\sum_{i=1}^t(1-(\mathbb{E}_{y_{< i}}[\mathbb{E}_{y_i}[\bm{w}_i\mid y_{<i}]])^2)\\
            &=\sum_{i=1}^t(1-\mathbb{E}_{y_{< i}}^2[\bm{g}_i])\\
            &\leq \sum_{i=1}^t(1-\frac{1}{|\mathcal{V}|})=t(1-\frac{1}{|\mathcal{V}|})
        \end{aligned}
        \label{eq:cal-for-mi-var-apx}
    \end{equation}
    Equality holds in equation~\eqref{eq:cal-for-mi-var-apx} if and only if $\forall 1\leq i\leq t:\mathbb{E}_{y_{< i}}[\bm{g}_i]=\frac{1}{|\mathcal{V}|}\mathbf{1}$.

    Considering equation~\eqref{eq:M-square-formula-apx}, we have
    \begin{align}
        &\quad\ \mathbb{E}_{y_{\leq t}}[\|\bm{M}_t\|_2^2]\notag\\
        &=\mathbb{E}_{y_{< t}}\left[\sum_{i=1}^t(1 - \|\bm{g}_i\|_2^2)\right]\notag\\
        &=\sum_{i=1}^t(1-\mathbb{E}_{y_{< t}}[\|\bm{g}_i\|_2^2])\label{eq:cal-for-M-apx}\\
        &=\sum_{i=1}^t(1-\mathbb{E}_{y_{< i}}[\|\bm{g}_i\|_2^2])\notag\\
        &\leq \sum_{i=1}^t(1-\frac{1}{|\mathcal{V}|})\notag\\
        &=t(1-\frac{1}{|\mathcal{V}|})\notag
    \end{align}

    Equality holds in equation~\eqref{eq:cal-for-M-apx} if and only if $\forall 1\leq i\leq t:\mathbb{E}_{y_{< i}}[\|\bm{g}_i\|_2^2]=\frac{1}{|\mathcal{V}|}$, i.e., $\forall 1\leq i\leq t:(\bm{g}_i=\frac{1}{|\mathcal{V}|}\mathbf{1}\ \mathrm{a.e.})$. Since $\forall 1\leq i\leq t:(\bm{g}_i=\frac{1}{|\mathcal{V}|}\mathbf{1}\ \mathrm{a.e.})\Rightarrow \forall 1\leq i\leq t:\mathbb{E}_{y_{< i}}[\bm{g}_i]=\frac{1}{|\mathcal{V}|}\mathbf{1}$, equation~\eqref{eq:entail-for-M-mi-var-max-apx} holds.
\end{proof}

\end{document}